\pdfoutput=1
\documentclass[twoside,11pt]{article}

\usepackage{jmlr2e}
\usepackage{amsmath}
\usepackage{url}

\newtheorem{thm}{Theorem}[section]
\newtheorem{cor}[thm]{Corollary}
\newtheorem{lem}[thm]{Lemma}

\newcommand{\norm}[1]{\left\Vert#1\right\Vert}

\newcommand{\tr}[1]{\textrm{tr}\left(#1\right)}
\newcommand{\R}{\mathbb{R}}
\newcommand{\D}{\mathcal{D}}
\newcommand{\eps}{\varepsilon}

\newcommand{\ona}{normalized-output algorithms}
\newcommand{\Var}{\textrm{Var}}
\newcommand{\Cov}{\textrm{Cov}}

\jmlrheading{??}{??}{??}{??}{??}{??}

\ShortHeadings{The Price of Normalization}{} \firstpageno{1}

\begin{document}

\title{Manifold Learning: The Price of Normalization}

\author{\name Yair Goldberg \email yairgo@cc.huji.ac.il \\
\addr Department of Statistics\\
The Hebrew University, 91905 Jerusalem, Israel\\
\AND
\name Alon Zakai \email alonzaka@pob.huji.ac.il\\
\addr Interdisciplinary Center for Neural Computation\\
The Hebrew University, 91905 Jerusalem, Israel\\
       \AND
       \name Dan Kushnir \email dan.kushnir@weizmann.ac.il \\
       \addr Department of Computer Science and Applied Mathematics \\
       The Weizmann Institute of Science, 76100 Rehovot, Israel\\
\AND
        \name Ya'acov Ritov \email yaacov.ritov@huji.ac.il \\
       \addr Department of Statistics\\
       The Hebrew University, 91905 Jerusalem, Israel\\
}

\editor{??}

\maketitle
\begin{abstract}
We analyze the performance of a class of manifold-learning
algorithms that find their output by minimizing a quadratic form
under some normalization constraints. This class consists of
Locally Linear Embedding (LLE), Laplacian Eigenmap, Local Tangent
Space Alignment (LTSA), Hessian Eigenmaps (HLLE), and Diffusion
maps. We present and prove conditions on the manifold that are
necessary for the success of the algorithms. Both the finite
sample case and the limit case are analyzed. We show that there
are simple manifolds in which the necessary conditions are
violated, and hence the algorithms cannot recover the underlying
manifolds. Finally, we present numerical results that demonstrate
our claims.
\end{abstract}


\section{Introduction}
Many seemingly complex systems described by high-dimensional data
sets are in fact governed by a surprisingly low number of
parameters. Revealing the low-dimensional representation of such
high-dimensional data sets not only leads to a more compact
description of the data, but also enhances our understanding of the
system. Dimension-reducing algorithms attempt to simplify the
system's representation without losing significant structural
information. Various dimension-reduction algorithms were developed
recently to perform embeddings for manifold-based data sets. These
include the following algorithms: Locally Linear
Embedding~\citep[LLE,][]{LLE}, Isomap~\citep{ISOMAP}, Laplacian
Eigenmaps~\citep[LEM,][]{belkin}, Local Tangent Space
Alignment~\citep[LTSA,][]{LTSA}, Hessian
Eigenmap~\citep[HLLE,][]{HessianEigenMap}, Semi-definite
Embedding~\citep[SDE,][]{Weinberger} and Diffusion
Maps~\citep[DFM,][]{DFM}.

These manifold-learning algorithms compute an embedding for some
given input. It is assumed that this input lies on a
low-dimensional manifold, embedded in some high-dimensional space.
Here a manifold is defined as a topological space that is locally
equivalent to a Euclidean space. It is further assumed that the
manifold is the image of a low-dimensional domain. In particular,
the input points are the image of a sample taken from the domain.
The goal of the manifold-learning algorithms is to recover the
original domain structure, up to some scaling and rotation. The
non-linearity of these algorithms allows them to reveal the domain
structure even when the manifold is not linearly embedded.

The central question that arises when considering the output of a
manifold-learning algorithm is, whether the algorithm reveals the
underlying low-dimensional structure of the manifold. The answer
to this question is not simple. First, one should define what
``revealing the underlying lower-dimensional description of the
manifold" actually means. Ideally, one could measure the degree of
similarity between the output and the original sample. However,
the original low-dimensional data representation is usually
unknown. Nevertheless, if the low-dimensional structure of the
data is known in advance, one would expect it to be approximated
by the dimension-reducing algorithm, at least up to some rotation,
translation, and global scaling factor. Furthermore, it would be
reasonable to expect the algorithm to succeed in recovering the
original sample's structure asymptotically, namely, when the
number of input points tends to infinity. Finally, one would hope
that the algorithm would be robust in the presence of noise.

Previous papers have addressed the central question posed earlier.
\citet{LTSA} presented some bounds on the local-neighborhoods'
error-estimation for LTSA. However, their analysis says nothing
about the global embedding. \citet{LTSAConvergence1} proved that,
asymptotically, LTSA recovers the original sample up to an affine
transformation. They assume in their analysis that the level of
noise tends to zero when the number of input points tends to
infinity. \citet{IsoMapConvergence} proved that, asymptotically,
the embedding given by the Isomap algorithm~\citep{ISOMAP}
recovers the geodesic distances between points on the manifold.

In this paper we develop theoretical results regarding the
performance of a class of manifold-learning algorithms, which
includes the following five algorithms: Locally Linear Embedding
(LLE), Laplacian Eigenmap (LEM), Local Tangent Space Alignment
(LTSA), Hessian Eigenmaps (HLLE), and Diffusion maps (DFM).

We refer to this class of algorithms as the \ona. The \ona~share a
common scheme for recovering the domain structure of the input
data set. This scheme is constructed in three steps. In the first
step, the local neighborhood of each point is found. In the second
step, a description of these neighborhoods is computed. In the
third step, a low-dimensional output is computed by solving some
convex optimization problem under some normalization constraints.
A detailed description of the algorithms is given in
Section~\ref{sec:algo}.

In Section~\ref{sec:mainExample} we discuss informally the criteria
for determining the success of manifold-learning algorithms. We show
that one should not expect the \ona~to recover geodesic distances or
local structures. A more reasonable criterion for success is a high
degree of similarity between the output of the algorithms and the
original sample, up to some affine transformation; the definition of
similarity will be discussed later. We demonstrate that under
certain circumstances, this high degree of similarity does not
occur. In Section~\ref{sec:grid} we find necessary conditions for
the successful performance of LEM and DFM on the two-dimensional
grid. This section serves as an explanatory introduction to the more
general analysis that appears in Section~\ref{sec:general}. Some of
the ideas that form the basis of the analysis in
Section~\ref{sec:grid} were discussed independently by both
\citet{gerber} and ourselves~\citep{ourICML}.
Section~\ref{sec:general} finds necessary conditions for the
successful performance of all the \ona~on general two-dimensional
manifolds. In Section~\ref{sec:asymptotics} we discuss the
performance of the algorithms in the asymptotic case. Concluding
remarks appear in Section~\ref{sec:discussion}. The detailed proofs
appear in the Appendix.

Our paper has two main results. First, we give well-defined
necessary conditions for the successful performance of the \ona.
Second, we show that there exist simple manifolds that do not
fulfill the necessary conditions for the success of the
algorithms. For these manifolds, the \ona~fail to generate output
that recovers the structure of the original sample. We show that
these results hold asymptotically for LEM and DFM. Moreover, when
noise, even of small variance, is introduced, LLE, LTSA, and HLLE
will fail asymptotically on some manifolds. Throughout the paper,
we present numerical results that demonstrate our claims.

\section{Description of Output-normalized Algorithms }\label{sec:algo}
In this section we describe in short the \ona. The presentation of
these algorithms is not in the form presented by the respective
authors. The form used in this paper emphasizes the similarities
between the algorithms and is better-suited for further derivations.
In Appendix~\ref{sec:appendixAlgorithms} we show the equivalence of
our representation of the algorithms and the representations that
appear in the original papers.

Let $X=[x_1,\ldots,x_N]', \, x_i\in \R^\D$ be the input data where
$\D$ is the dimension of the ambient space and $N$ is the size of
the sample. The \ona~attempt to recover the underlying structure
of the input data $X$ in three steps.

In the first step, the \ona~assign neighbors to each input point
$x_i$ based on the Euclidean distances in the high-dimensional
space\footnote{The neighborhoods are not mentioned explicitly
by~\citet{DFM}. However, since a sparse optimization problem is
considered, it is assumed implicitly that neighborhoods are
defined (see Sec.~2.7 therein).}. This can be done, for example,
by choosing all the input points in an $r$-ball around $x_i$ or
alternatively by choosing $x_i$'s $K$-nearest-neighbors. The
neighborhood of $x_i$ is given by the matrix
$X_i=[x_i,x_{i,1},\ldots,x_{i,K}]'$ where $x_{i,j}:\,j=1,\ldots,K$
are the neighbors of $x_i$. Note that $K=K(i)$ can be a function
of $i$, the index of the neighborhood, yet we omit this index to
simplify the notation. For each neighborhood, we define the radius
of the neighborhood as
\begin{equation}\label{eq:radius_of_neighborhood}
r(i)=\max_{j,k\in\{0,\ldots,K\}}\norm{x_{i,j}-x_{i,k}}
\end{equation}
where we define $x_{i,0}=x_i$. Finally, we assume throughout this
paper that the neighborhood graph is connected.

In the second step, the \ona~compute a description of the local
neighborhoods that were found in the previous step. The
description of the $i$-th neighborhood is given by some weight
matrix $W_i$. The matrices $W_i$ for the different algorithms are
presented.
\begin{itemize}
  \item LEM and DFM: $W_i$ is a $K\times (K+1)$ matrix,
  \begin{equation*}
    W_i= \left(
         \begin{array}{ccccc}
          w_{i,1}^{1/2} & -w_{i,1}^{1/2} & 0 &\cdots  & 0  \\
           w_{i,2}^{1/2} & 0 & -w_{i,2}^{1/2}& \ddots & \vdots \\
           \vdots & \vdots & \ddots & \ddots & 0 \\
           w_{i,K}^{1/2} &0 & \cdots & 0 & -w_{i,K}^{1/2} \\
         \end{array}
       \right)\;.
 \end{equation*}

For LEM $w_{i,j}=1$ is a natural choice, yet it is also possible
to define the weights as $\tilde{w}_{i,j}=e^{-\norm{x_i -
x_{i,j}}^2/\eps}$, where $\varepsilon$ is the width parameter of
the kernel. For the case of DFM,
\begin{equation}\label{eq:w_ijForDFM}
    w_{i,j}=\frac{k_{\eps}(x_i,x_{i,j})}{q_{\eps}(x_i)^{\alpha}q_{\eps}(x_{i,j})^{\alpha}}\,,
\end{equation}
where $k_{\varepsilon}$ is some rotation-invariant kernel,
$q_{\eps}(x_i)=\sum_j k_{\eps}(x_i,x_{i,j})$ and $\eps$ is again a
width parameter. We will use $\alpha=1$ in the normalization of
the diffusion kernel, yet other values of $\alpha$ can be
considered \citep[see details in][]{DFM}. For both LEM and DFM, we
define the matrix $D$ to be a diagonal matrix where
$d_{ii}=\sum_{j}w_{i,j}$.

\item LLE: $W_i$ is a $1\times (K+1)$ matrix,
  \begin{equation*}
    W_i=\left(
        \begin{array}{cccc}
          1 & -w_{i,1} & \cdots &  -w_{i,K} \\
        \end{array}
      \right)\,.
\end{equation*}
The weights $w_{i,j}$ are chosen so that $x_i$ can be best linearly
reconstructed from its neighbors. The weights minimize the
reconstruction error function
\begin{equation}\label{eq:reconstructionError}
    \Delta^i\left(w_{i,1},\ldots,w_{i,K}\right)=\|x_i-\sum_j w_{i,j} x_{i,j}\| ^2
\end{equation}
under the constraint $\sum_j w_{i,j}=1$. In the case where there
is more than one solution that minimizes $\Delta^i$,
regularization is applied to force a unique solution~\citep[for
details, see][]{thinkGlobally}.
  \item LTSA: $W_i$ is a $(K+1)\times (K+1)$ matrix,
    \begin{equation*}
    W_i=(I-P_i {P_i}')H\,.
  \end{equation*}
  Let $U_i L_i {V_i}'$ be the SVD of
  $X_i-\textbf{1}{\bar{x}_i}'$ where $\bar{x}_i$ is the sample mean of
  $X_i$ and $\textbf{1}$ is a vector of ones~\citep[for details about SVD, see, for example,][]{MatrixComputations}.
  Let $P_i=[u_{(1)},\ldots,u_{(d)}]$ be the matrix that holds the first $d$ columns of
  $U_i$ where $d$ is the output dimension. The matrix
  $H=I-\frac{1}{K}\textbf{1}\textbf{1}'$ is the centering matrix. See also~\citet{LTSAConvergence1}
  regarding this representation of the algorithm.

  \item HLLE: $W_i$ is a $d(d+1)/2 \times (K+1)$ matrix,
  \begin{equation*}
    W_i=(\textbf{0}, H^i)
  \end{equation*}
  where $\textbf{0}$ is a vector of
  zeros and $H^i$ is the $\frac{d(d+1)}{2}\times K$ Hessian
  estimator.\\
  The estimator can be calculated as follows. Let $U_i L_i {V_i}'$ be the SVD of
  $X_i-\textbf{1}{\bar{x}_i}'$. Let
  \begin{equation*}
    M_i=[\textbf{1},U_i^{(1)},\ldots,U_i^{(d)},\textrm{diag}(U_i^{(1)}U_i^{(1)}{}'),
    \textrm{diag}(U_i^{(1)}U_i^{(2)}{}'),\ldots,
    \textrm{diag}(U_i^{(d)}U_i^{(d)}{}')]\,,
  \end{equation*}
  where the operator diag returns a column vector
  formed from the diagonal elements of the matrix.
  Let $\widetilde{M}_i$ be the result of the Gram-Schmidt orthonormalization on $M_i$. Then
  $H^i$ is defined as the transpose of the last $d(d+1)/2$ columns
  of $\widetilde{M}_i$.
    \end{itemize}

The third step of the \ona~is to find a set of points
$Y=[y_1,\ldots,y_N]', \, y_i\in \R^d $ where $d\leq \D$ is the
dimension of the manifold. $Y$ is found by minimizing a convex
function under some normalization constraints, as follows. Let $Y$
be any $N\times d $ matrix. We define the $i$-th neighborhood matrix
$Y_i=[y_i,y_{i,1},\ldots,y_{i,K}]'$ using the same pairs of indices
$i,j$ as in $X_i$. The cost function for all of the \ona~is given by
\begin{equation}\label{eq:cost_function}
\Phi(Y)=\sum_{i=1}^N\phi(Y_i)=\sum_{i=1}^N \norm{W_i Y_i}^2_F\,,
\end{equation}
under the normalization constraints
\begin{equation}\label{eq:normalization_constraints}
    \left\{\begin{array}{c}
            Y'DY=I \\
            Y'D\textbf{1}=\textbf{0}
          \end{array}\right.\;\mathrm{for\,\,LEM\,\,and\,\,DFM,\;\;}
    \left\{\begin{array}{c}
            \Cov(Y)=I \\
            Y'\textbf{1}=\textbf{0}
          \end{array}\right.\mathrm{\;for\,\,LLE,\,\,LTSA\,\,and\,\,HLLE,}
\end{equation}
where $\norm{\;}_F$ stands for the Frobenius norm, and $W_i$ is
algorithm-dependent.

Define the output matrix $Y$ to be the matrix that achieves the
minimum of $\Phi$ under the normalization constraints of
Eq.~\ref{eq:normalization_constraints} ($Y$ is defined up to
rotation). Then we have the following: the embeddings of LEM and LLE
are given by the according output matrices $Y$; the embeddings of
LTSA and HLLE are given by the according output matrices
$\frac{1}{\sqrt{N}}Y$; and the embedding of DFM is given by a linear
transformation of $Y$ as discussed in
Appendix~\ref{sec:appendixAlgorithms}. The discussion of the
algorithms' output in this paper holds for any affine transformation
of the output (see Section~\ref{sec:mainExample}). Thus, without
loss of generality, we prefer to discuss the output matrix $Y$
directly, rather than the different embeddings. This allows a
unified framework for all five normalized-output algorithms.

\section{Embedding quality}\label{sec:mainExample}
In this section we discuss possible definitions of ``successful
performance" of manifold-learning algorithms. To open our
discussion, we present a numerical example. We chose to work with
LTSA rather arbitrarily. Similar results can be obtained using the
other algorithms.

\begin{figure}[t]
\vskip 0.2in
\begin{center}
\includegraphics{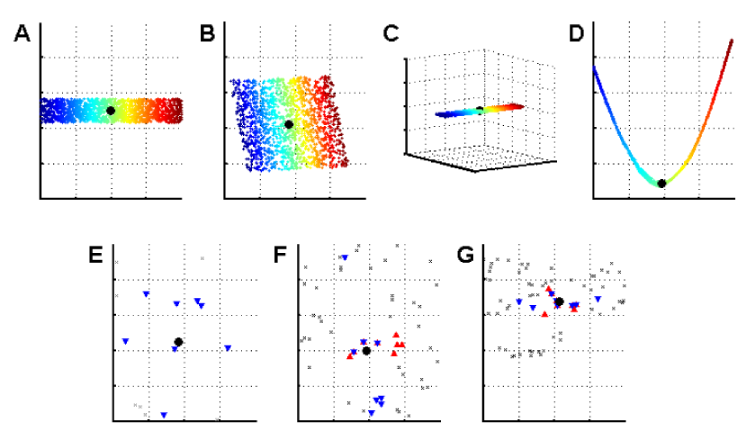}
\caption{The output of LTSA~(B) for the (two-dimensional) input
shown in~(A), where the input is a uniform sample from the strip
$[0,1]\times [0,6]$. Ideally one would expect the two to be
identical. The normalization constraint shortens the horizontal
distances and lengthens the vertical distances, leading to the
distortion of geodesic distances. (E) and (F) focus on the points
shown in black in (A) and (B), respectively. The blue triangles in
(E) and (F) are the $8$-nearest-neighborhood of the point denoted by
the full black circle. The red triangles in (F) indicate the
neighborhood computed for the corresponding point (full black
circle) in the output space. Note that less than half of the
original neighbors of the point remain neighbors in the output
space. The input (A) with the addition of Gaussian noise normal to
the manifold and of variance $10^{-4}$ is shown in~(C). The output
of LTSA for the noisy input is shown in~(D). (G) shows a closeup of
the neighborhood of the point indicated by the black circle in~(D).}
\label{fig:LTSAexample}
\end{center}
\vskip -0.2in
\end{figure}

The example we consider is a uniform sample from a two-dimensional
strip, shown in Fig.~\ref{fig:LTSAexample}A. Note that in this
example, $\D=d$; i.e., the input data is identical to the original
data. Fig.~\ref{fig:LTSAexample}B presents the output of LTSA on
the input in Fig.~\ref{fig:LTSAexample}A. The most obvious
difference between input and output is that while the input is a
strip, the output is roughly square. While this may seem to be of
no importance, note that it means that the algorithm, like all the
\ona, does not preserve geodesic distances even up to a scaling
factor. By definition, the geodesic distance between two points on
a manifold is the length of the shortest path on the manifold
between the two points. Preservation of geodesic distances is
particularly relevant when the manifold is isometrically embedded.
In this case, assuming the domain is convex, the geodesic distance
between any two points on the manifold is equal to the Euclidean
distance between the corresponding domain points. Geodesic
distances are conserved, for example, by the Isomap
algorithm~\citep{ISOMAP}.

Figs.~\ref{fig:LTSAexample}E and \ref{fig:LTSAexample}F present
closeups of Figs.~\ref{fig:LTSAexample}A and
\ref{fig:LTSAexample}B, respectively. Here, a less obvious
phenomenon is revealed: the structure of the local neighborhood is
not preserved by LTSA. By local structure we refer to the angles
and distances (at least up to a scale) between all points within
each local neighborhood. Mappings that preserve local structures
up to a scale are called conformal mappings \citep[see for
example][]{cIsomap,ShaExtensionSpectralMethods}. In addition to
the distortion of angles and distances, the $K$-nearest-neighbors
of a given point on the manifold do not necessarily correspond to
the $K$-nearest-neighbors of the respective output point, as shown
in Figs.~\ref{fig:LTSAexample}E and \ref{fig:LTSAexample}F.
Accordingly, we conclude that the original structure of the local
neighborhoods is not necessarily preserved by the \ona.

The above discussion highlights the fact that one cannot expect
the \ona~to preserve geodesic distances or local neighborhood
structure. However, it seems reasonable to demand that the output
of the \ona~resemble an affine transformation of the original
sample. In fact, the output presented in
Fig.~\ref{fig:LTSAexample}B is an affine transformation of the
input, which is the original sample, presented in
Fig.~\ref{fig:LTSAexample}A. A formal similarity criterion based
on affine transformations is given by~\citet{LTSAConvergence1}. In
the following, we will claim that a normalized-output algorithm
succeeds (or fails) based on the existence (or lack thereof) of
resemblance between the output and the original sample, up to an
affine transformation.

Fig.~\ref{fig:LTSAexample}D presents the output of LTSA on a noisy
version of the input, shown in Fig.~\ref{fig:LTSAexample}C. In
this case, the algorithm prefers an output that is roughly a
one-dimensional curve embedded in $\R^2$. While this result may
seem incidental, the results of all the other \ona~for this
example are essentially the same.

Using the affine transformation criterion, we can state that LTSA
succeeds in recovering the underlying structure of the strip shown
in Fig.~\ref{fig:LTSAexample}A. However, in the case of the noisy
strip shown in Fig.~\ref{fig:LTSAexample}C, LTSA fails to recover
the structure of the input. We note that all the other \ona~perform
similarly.

For practical purposes, we will now generalize the definition of
failure of the \ona. This definition is more useful when it is
necessary to decide whether an algorithm has failed, without
actually computing the output. This is useful, for example, when
considering the outputs of an algorithm for a class of manifolds.

We now present the generalized definition of failure of the
algorithms. Let $X=X_{N\times d}$ be the original sample. Assume
that the input is given by $\psi(X)\subset \R^\D$, where
$\psi:\R^d\rightarrow\R^\D$ is some smooth function, and $\D\geq d$
is the dimension of the input. Let $Y=Y_{N\times d}$ be an affine
transformation of the original sample $X$, such that the
normalization constraints of Eq.~\ref{eq:normalization_constraints}
hold. Note that $Y$ is algorithm-dependent, and that for each
algorithm, $Y$ is unique up to rotation and translation. When the
algorithm succeeds it is expected that the output will be similar to
a normalized version of $X$, namely to $Y$.  Let $Z=Z_{N\times d}$
be any matrix that satisfies the same normalization constraints. We
say that the algorithm has failed if $\Phi(Y)>\Phi(Z)$, and $Z$ is
substantially different from $Y$, and hence also from $X$. In other
words, we say that the algorithm has failed when a substantially
different embedding $Z$ has a lower cost than the most appropriate
embedding $Y$. A precise definition of ``substantially different" is
not necessary for the purposes of this paper. It is enough to
consider $Z$ substantially different from $Y$ when $Z$ is of lower
dimension than $Y$, as in Fig.~\ref{fig:LTSAexample}D.

We emphasize that the matrix $Z$ is not necessarily similar to the
output of the algorithm in question. It is a mathematical
construction that shows when the output of the algorithm is not
likely to be similar to $Y$, the normalized version of the true
manifold structure. The following lemma shows that if
$\Phi(Y)>\Phi(Z)$, the inequality is also true for a small
perturbation of $Y$. Hence, it is not likely that an output that
resembles $Y$ will occur when $\Phi(Y)>\Phi(Z)$ and $Z$ is
substantially different from $Y$.

\begin{lem}\label{lem:criterion_failure}
Let $Y$ be an $N\times d$ matrix. Let $\widetilde{Y}=Y+\eps E$ be a
perturbation of $Y$, where $E$ is an $N\times d$ matrix such that
$\norm{E}_F=1$ and where $\eps>0$. Let $S$ be the maximum number of
neighborhoods to which a single input point belongs. Then for LLE
with positive weights $w_{i,j}$, LEM, DFM, LTSA, and HLLE, we have
\begin{equation*}
    \Phi(\widetilde{Y})>(1-4\eps)\Phi(Y)-4\eps C_a S \,,
\end{equation*}
where $C_a$ is a constant that depends on the algorithm.
\end{lem}
The use of positive weights in LLE is discussed
in~\citet[Section~5]{thinkGlobally}; a similar result for LLE with
general weights can be obtained if one allows a bound on the values
of $w_{i,j}$. The proof of Lemma~\ref{lem:criterion_failure} is
given in Appendix~\ref{sec:appendixLemCriterionFailure}.

\section{Analysis of the two-dimensional grid}\label{sec:grid}
In this section we analyze the performance of LEM on the
two-dimensional grid. In particular, we argue that LEM cannot
recover the structure of a two-dimensional grid in the case where
the aspect ratio of the grid is greater than $2$. Instead, LEM
prefers a one-dimensional curve in $\R^2$. Implications also
follow for DFM, as explained in
Section~\ref{sec:ImplicationToDFM}, followed by a discussion of
the other \ona. Finally, we present empirical results that
demonstrate our claims.

In Section~\ref{sec:general} we prove a more general statement
regarding any two-dimensional manifold. Necessary conditions for
successful performance of the \ona~on such manifolds are
presented. However, the analysis in this section is important in
itself for two reasons. First, the conditions for the success of
LEM on the two-dimensional grid are more limiting. Second, the
analysis is simpler and points out the reasons for the failure of
all the \ona~when the necessary conditions do not hold.

\subsection{Possible embeddings of a two-dimensional grid}\label{sec:two_embeddings}
We consider the input data set $X$ to be the two-dimensional grid
$[-m,\ldots,m]\times[-q,\ldots,q]$, where $m \geq q$. We denote
$x_{ij}=(i,j)$. For convenience, we regard $X=(X^{(1)},X^{(2)})$
as an $N\times 2$ matrix, where $N=(2m+1)(2q+1)$ is the number of
points in the grid. Note that in this specific case, the original
sample and the input are the same.

In the following we present two different embeddings, $Y$ and $Z$.
Embedding $Y$ is the grid itself, normalized so that
$\mathrm{Cov}(Y)=I$. Embedding $Z$ collapses each column to a point
and positions the resulting points in the two-dimensional plane in a
way that satisfies the constraint $\mathrm{Cov}(Z)=I$ (see
Fig.~\ref{fig:embeddings} for both). The embedding $Z$ is a curve in
$\R^2$ and clearly does not preserve the original structure of the
grid.

\begin{figure}[t]
\vskip 0.2in
\begin{center}
\includegraphics[width=4.2in]{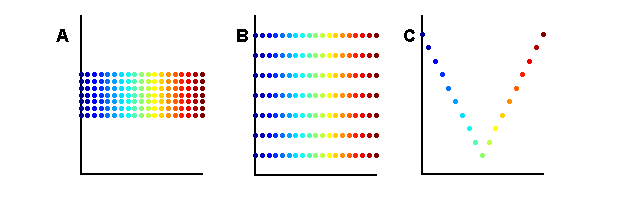}
\caption{(A) The input grid. (B) Embedding $Y$, the normalized grid.
(C) Embedding $Z$, a curve that satisfies $\mathrm{Cov}(Z)=I$.}
\label{fig:embeddings}
\end{center}
\vskip -0.2in
\end{figure}

We first define  the embeddings more formally. We start by defining
$\widehat{Y}=X(X'DX)^{-1/2}$. Note that this is the only linear
transformation of $X$ (up to rotation) that satisfies the conditions
$\widehat{Y}'D\textbf{1}=\textbf{0}$ and
$\widehat{Y}'D\widehat{Y}=I$, which are the normalization
constraints for LEM (see Eq.~\ref{eq:normalization_constraints}).
However, the embedding $\widehat{Y}$ depends on the matrix $D$,
which in turn depends on the choice of neighborhoods. Recall that
the matrix $D$ is a diagonal matrix, where $d_{ii}$ equals the
number of neighbors of the $i$-th point. Choose $r$ to be the radius
of the neighborhoods. Then, for all inner points $x_{ij}$, the
number of neighbors $K(i,j)$ is a constant, which we denote as $K$.
We shall call all points with less than $K$ neighbors
\textit{boundary points}. Note that the definition of boundary
points depends on the choice of $r$. For inner points of the grid we
have $d_{ii}\equiv K$. Thus, when $K\ll N$ we have $X'DX\approx
KX'X$.

We define $Y=X\Cov(X)^{-1/2}$. Note that $Y'\textbf{1}=0$,
$\Cov(Y)=I$ and for $K\ll N$, $Y\approx \sqrt{KN}\widehat{Y}$. In
this section we analyze the embedding $Y$ instead of $\widehat{Y}$,
thereby avoiding the dependence on the matrix $D$ and hence
simplifying the notation. This simplification does not significantly
change the problem and does not affect the results we present.
Similar results are obtained in the next section for general
two-dimensional manifolds, using the exact normalization constraints
(see Section~\ref{sec:embeddings_LEM_DFM}).

Note that $Y$ can be described as the set of points $[-m/\sigma,
\ldots ,m/\sigma]\times[-q/\tau, \ldots, q/\tau]$, where
$y_{ij}=(i/\sigma,j/\tau)$. The constants
$\sigma^2=\textrm{Var}(X^{(1)})$ and $\tau^2=\textrm{Var}(X^{(2)})$
ensure that the normalization constraint $\mathrm{Cov}(Y)=I$ holds.
Straightforward computation (see Appendix~\ref{sec:appendixSigma})
shows that
\begin{equation}\label{eq:sigma}
\sigma^2=\frac{(m+1)m}{3}\,;\;\tau^2=\frac{(q+1)q}{3} \,.
\end{equation}

The definition of the embedding $Z$ is as follows:
\begin{equation}\label{eq:z_i}
    z_{ij}=\left\{ \begin{array}{cc}
                \left(\frac{i}{\sigma},\frac{-2i}{\rho}-\bar{z}^{(2)}\right) & i
                \leq 0\\
                 \\
                 \left(\frac{i}{\sigma},\frac{2i}{\rho}-\bar{z}^{(2)}\right) & i\geq 0
               \end{array}\right.\,,
\end{equation}
where $\bar{z}^{(2)}=\frac{(2q+1)2}{N\rho}\sum_{i=1}^{m}(2i)$
ensures that $Z'\textbf{1}=\textbf{0}$, and $\sigma$ (the same
$\sigma$ as before; see below) and $\rho$ are chosen so that
sample variance of $Z^{(1)}$ and $Z^{(2)}$ is equal to one. The
symmetry of $Z^{(1)}$ about the origin implies that
$\textrm{Cov}(Z^{(1)},Z^{(2)})=0$, hence the normalization
constraint $\textrm{Cov}(Z)=I$ holds. $\sigma$ is as defined in
Eq.~\ref{eq:sigma}, since $Z^{(1)}=Y^{(1)}$ (with both defined
similarly to $X^{(1)}$). Finally, note that the definition of
$z_{ij}$ does not depend on $j$.

\subsection{Main result for LEM on the two-dimensional grid}
We estimate $\Phi(Y)$ by $N \phi(Y_{ij})$ (see
Eq.~\ref{eq:cost_function}), where $y_{ij}$ is an inner point of
the grid and $Y_{ij}$ is the neighborhood of $y_{ij}$; likewise,
we estimate $\Phi(Z)$ by $N \phi(Z_{ij})$ for an inner point
$z_{ij}$. For all inner points, the value of $\phi(Y_{ij})$ is
equal to some value $\phi$. For boundary points, $\phi(Y_{ij})$ is
bounded by $\phi$ multiplied by some constant that depends only on
the number of neighbors. Hence, for large $m$ and $q$, the
difference between $\Phi(Y)$ and $N \phi(Y_{ij})$ is negligible.

The main result of this section states:
\begin{thm}\label{thm:grid}
Let $y_{ij}$ be an inner point and let the ratio $\frac{m}{q}$ be
greater than $2$. Then
\begin{equation*}
\phi(Y_{ij})>\phi(Z_{ij})
\end{equation*}
for neighborhood-radius $r$ that satisfies $1\leq r \leq 3$, or
similarly, for $K$-nearest neighborhoods where $K=4,8,12$.
\end{thm}
This indicates that for aspect ratios $\frac{m}{q}$ that are greater
than $2$ and above, mapping $Z$, which is essentially
one-dimensional, is preferred to $Y$, which is a linear
transformation of the grid. The case of general $r$-ball
neighborhoods is discussed in Appendix~\ref{sec:appendixRball} and
indicates that similar results should be expected.

The proof of the theorem is as follows. It can be shown analytically
(see Fig.~\ref{fig:penalty}) that
\begin{equation}\label{eq:PhiYestimation}  \phi(Y_{ij})=F(K)\left(\frac{1}{\sigma^2}+\frac{1}{\tau^2}\right)
\,,
\end{equation}
where
\begin{equation}\label{eq:F}
F(4)=2\,;\;\;F(8)=6\,;\;\; F(12)=14\,.
\end{equation}
For higher $K$, $F(K)$ can be approximated for any $r$-ball
neighborhood of $y_{ij}$ (see Appendix~\ref{sec:appendixRball}).

\begin{figure}[t]
\vskip 0.2in
\begin{center}
\includegraphics[width=3.5in]{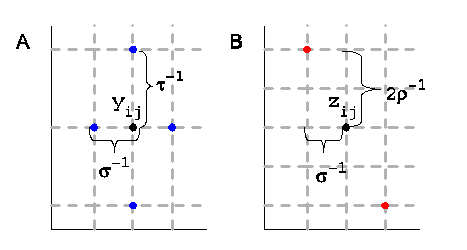}
\caption{ (A) The normalized grid at an inner point $y_{ij}$. The
$4$-nearest-neighbors of $y_{ij}$ are marked in blue. Note that the
neighbors from the left and from the right are at a distance of
$1/\sigma$, while the neighbors from above and below are at a
distance of $1/\tau$. The value of $\phi(Y_{ij})$ is equal to the
sum of squared distances of $y_{ij}$ to its neighbors. Hence, we
obtain that $\phi(Y_{ij})=2/\sigma^2+2/\tau^2$ when $K=4$ and
$\phi(Y_{ij})=2/\sigma^2+2/\tau^2+4(1/\sigma^2+1/\tau^2)$
when $K=8$. 
(B) The curve embedding at an inner point $z_{ij}$. The neighbors
of $z_{ij}$ from the left and from the right are marked in red.
The neighbors from above and below are embedded to the same point
as $z_{ij}$. Note that the squared distance between $z_{ij}$ and
$z_{(i\pm1)j}$ equals $1/\sigma^2+4/\rho^2$. Hence, $\phi(Z_{ij})=
2(1/\sigma^2+4/\rho^2)$ when $K=4$, and $\phi(Z_{ij})=
6(1/\sigma^2+4/\rho^2)$ when $K=8$.}\label{fig:penalty}
\end{center}
\vskip -0.2in
\end{figure}


It can be shown (see Fig.~\ref{fig:penalty}) that
\begin{equation}\label{eq:PhiZestimation}
 \phi(Z_{ij})=\widetilde{F}(K)\left(\frac{1}{\sigma^2}+\frac{4}{\rho^2}\right)
\,,
\end{equation}
where $\widetilde{F}(K) = F(K)$ for $K=4,8,12$. For higher $K$, it
can be shown  (see Appendix~\ref{sec:appendixRball}) that
$\widetilde{F}(K)\approx F(K)$ for any $r$-ball neighborhood.


A careful computation (see Appendix~\ref{sec:appendixRho}) shows
that
\begin{equation}\label{eq:rhoGEQsigma}
\rho >\sigma \,,
\end{equation}
and therefore
\begin{equation}\label{eq:PhiZestimation2}
\phi(Z_{ij})<\frac{ 5F(K)}{\sigma^2} \,.
\end{equation}

Assume that $\frac{m}{q}>2$. Since both $m$ and $q$ are integers, we
have that $m+1\geq 2(q+1)$. Hence, using Eq.~\ref{eq:sigma} we have
\begin{equation*}
    \sigma^2=\frac{m(m+1)}{3}>\frac{4q(q+1)}{3}=4\tau^2\,.
\end{equation*}
Combining this result with Eqs.~\ref{eq:PhiYestimation} and
\ref{eq:PhiZestimation2} we have
\begin{equation*}\frac{m}{q}>2
\Rightarrow \phi(Y_{ij})>\phi(Z_{ij}) \,.
\end{equation*}
which proves Theorem~\ref{thm:grid}.

\subsection{Implications to other algorithms}\label{sec:ImplicationToDFM}
We start with implications regarding DFM. There are two main
differences between LEM and DFM. The first difference is the
choice of the kernel. LEM chooses $w_{i,j}=1$, which can be
referred to as the ``window" kernel~\citep[a Gaussian weight
function was also considered by][]{belkin}. DFM allows a more
general rotation-invariant kernel, which includes the ``window"
kernel of LEM. The second difference is that DFM renormalizes the
weights $k_{\eps}(x_i,x_{i,j})$ (see Eq.~\ref{eq:w_ijForDFM}).
However, for all the inner points of the grid with neighbors that
are also inner points, the renormalization factor
$(q_{\eps}(x_i)^{-1}q_{\eps}(x_{i,j})^{-1})$ is a constant.
Therefore, if DFM chooses the ``window" kernel, it is expected to
fail, like LEM. In other words, when DFM using the ``window''
kernel is applied to a grid with aspect ratio slightly greater
than $2$ or above, DFM will prefer the embedding $Z$ over the
embedding $Y$ (see Fig~\ref{fig:embeddings}). For a more general
choice of kernel, the discussion in
Appendix~\ref{sec:appendixRball} indicates that a similar failure
should occur. This is because the relation between the estimations
of $\Phi(Y)$ and $\Phi(Z)$ presented in
Eqs.~\ref{eq:PhiYestimation} and~\ref{eq:PhiZestimation} holds for
any rotation-invariant kernel (see
Appendix~\ref{sec:appendixRball}). This observation is also
evident in numerical examples, as shown in Figs.~\ref{fig:grid}
and \ref{fig:random}.

In the cases of LLE with no regularization, LTSA, and HLLE, it can
be shown that $\Phi(Y)\equiv 0$. Indeed, for LTSA and HLLE, the
weight matrix $W_i$ projects on a space that is perpendicular to the
SVD of the neighborhood $X_i$, thus $\norm{W_i X_i}_F^2=0$. Since
$Y_i=X_i\Cov(X)^{-1/2}$, we have $\norm{W_i Y_i}_F^2=0$, and,
therefore, $\Phi(Y)\equiv 0$. For the case of LLE with no
regularization, when $K\geq 3$, each point can be reconstructed
perfectly from its neighbors, and the result follows. Hence, a
linear transformation of the original data should be the preferred
output. However, the fact that $\Phi(Y)\equiv 0$ relies heavily on
the assumption that both the input $X$ and the output $Y$ are of the
same dimension (see Theorem~\ref{thm:general} for manifolds embedded
in higher dimensions), which is typically not the case in
dimension-reducing applications.

\subsection{Numerical results}\label{sec:numericalExamplesGrid}
For the following numerical results, we used the Matlab
implementation written by the respective algorithms' authors as
provided by~\cite{ManifoldHompage} (a minor correction was applied
to the code of HLLE).

We ran the LEM algorithm on data sets with aspect ratios above and
below $2$. We present results for both a grid and a uniformly
sampled strip.  The neighborhoods were chosen using $K$-nearest
neighbors with $K=4,8,16$, and $64$. We present the results for
$K=8$; the results for $K=4,16$, and $64$ are similar. The results
for the grid and the random sample are presented in
Figs.~\ref{fig:grid} and~\ref{fig:random}, respectively.

We ran the DFM algorithm on the same data sets. 
We used the normalization constant $\alpha = 1$ and the kernel
width $\sigma = 2$; the results for $\sigma = 1,4$, and $8$ are
similar. The results for the grid and the random sample are
presented in Figures~\ref{fig:grid} and~\ref{fig:random},
respectively.

\begin{figure}[!t]
\vskip 0.2in
\begin{center}
\setlength{\epsfxsize}{3.25in}
\includegraphics[width=0.6\linewidth]{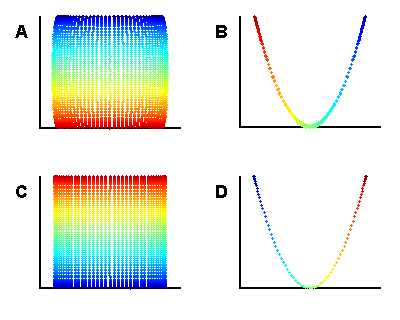}
\caption{The output of LEM on a grid of dimensions $81\times 41$
is presented in~(A). The result of LEM for the grid of dimensions
$81\times 39$ is presented in~(B). The number of neighbors in both
computations is $8$. The output for DFM on the same data sets
using $\sigma=2$ appears in (C) and (D), respectively.}
\label{fig:grid}
\end{center}
\vskip -0.2in
\end{figure}

\begin{figure}[!th]
\vskip 0.2in
\begin{center}
\setlength{\epsfxsize}{3.25in}
\includegraphics[width=0.8\linewidth]{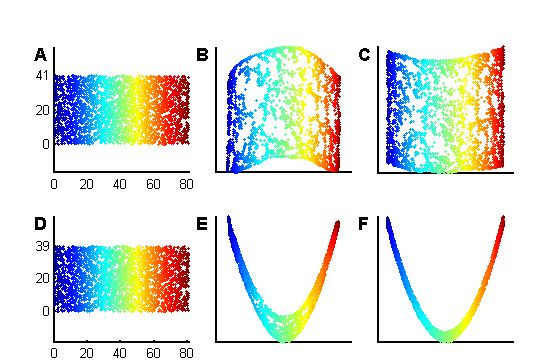}
\caption{(A) and (D) show the same 3000 points, uniformly-sampled
from the unit square, scaled to the areas $[0,81]\times[0,41]$ and
$[0,81]\times[0,39]$, respectively. (B) and (E) show the outputs
of LEM for inputs (A) and (D), respectively. The number of
neighbors is both computations is $8$. (C) and (F) show the output
for DFM on the same data sets using $\sigma=2$. Note the sharp
change in output structure for extremely similar inputs.}
\label{fig:random}
\end{center}
\vskip -0.2in
\end{figure}

Both examples clearly demonstrate that for aspect ratios
sufficiently greater than $2$, both LEM and DFM prefer a solution
that collapses the input data to a nearly one-dimensional output,
normalized in $\R^2$. This is exactly as expected, based on our
theoretical arguments.

Finally, we ran LLE, HLLE, and LTSA on the same data sets. In the
case of the grid, both LLE and LTSA (roughly) recovered the grid
shape for $K=4,8,16$, and $64$, while HLLE failed to produce any
output due to large memory requirements. In the case of the random
sample, both LLE and HLLE succeeded for $K=16,64$ but failed for
$K=4,8$. LTSA succeeded for $K=8,16$, and $64$ but failed for $K=4$.
The reasons for the failure for lower values of $K$ are not clear,
but may be due to roundoff errors. In the case of LLE, the failure
may also be related to the use of regularization in LLE's second
step.

\section{Analysis for general two-dimensional manifolds}\label{sec:general}
The aim of this section is to present necessary conditions for the
success of the \ona~on general two-dimensional manifolds embedded
in high-dimensional space. We show how this result can be further
generalized to manifolds of higher dimension. We demonstrate the
theoretical results using numerical examples.

\subsection{Two different embeddings for a two-dimensional manifold}
We start with some definitions. Let $X=[x_1,\ldots,x_N]', \, x_i\in
\R^2$ be the original sample. Without loss of generality, we assume
that
\begin{equation*}
    \bar{x}=\mathbf{0};\;\;\;\;\;\mathrm{Cov}(X)\equiv \Sigma =\left(
                                    \begin{array}{cc}
                                      \sigma^2 & 0 \\
                                      0 & \tau^2 \\
                                    \end{array}
                                  \right)\,.
\end{equation*}
As in Section~\ref{sec:grid}, we assume that $\sigma>\tau$. Assume
that the input for the \ona~is given by $\psi(X)\subset \R^\D$
where $\psi:\R^2\rightarrow\R^\D$ is a smooth function and $\D\geq
2$ is the dimension of the input. When the mapping $\psi$ is an
isometry, we expect $\Phi(X)$ to be small. We now take a close
look at $\Phi(X)$.
\begin{eqnarray*}
  \Phi(X) = \sum_{i=1}^N \norm{W_i X_i}^2_F=\sum_{i=1}^N \norm{W_i X_i^{(1)}}^2
  +\sum_{i=1}^N \norm{W_i X_i^{(2)}}^2\,,
\end{eqnarray*}
where $X_i^{(j)}$ is the $j$-th column of the neighborhood $X_i$.
Define $e_i^{(j)}= \norm{W_i X_i^{(j)}}^2$, and note that
$e_i^{(j)}$ depends on the different algorithms through the
definition of the matrices $W_i$. The quantity $e_i^{(j)}$ is the
portion of error obtained by using the $j$-th column of the $i$-th
neighborhood when using the original sample as output. Denote
$\bar{e}^{(j)}=\frac{1}{N}\sum_i e_i^{(j)}$, the average error
originating from the $j$-th column.

We define two different embeddings for $\psi(X)$, following the
logic of Sec.~\ref{sec:two_embeddings}. Let
\begin{equation}\label{eq:y_i_general}
    Y=X\Sigma^{-1/2}
\end{equation}
be the first embedding. Note that $Y$ is just the original sample
up to a linear transformation that ensures that the normalization
constraints $\Cov(Y)=I$ and $Y'\textbf{1}=\textbf{0}$ hold.
Moreover, $Y$ is the only transformation of $X$ that satisfies
these conditions, which are the normalization constraints for LLE,
HLLE, and LTSA. In Section~\ref{sec:embeddings_LEM_DFM} we discuss
the modified embeddings for LEM and DFM.

The second embedding, $Z$, is given by
\begin{equation}\label{eq:z_i_general}
   z_{i}=\left\{ \begin{array}{cc}
                 \left(\frac{x_i^{(1)}}{\sigma},
                 \frac{- x_i^{(1)}}{\rho}-\bar{z}^{(2)}\right) & x_i^{(1)}
                 <0\\
                 \\
                 \left(\frac{x_i^{(1)}}{\sigma},
                 \frac{\kappa x_i^{(1)}}{\rho}-\bar{z}^{(2)}\right) & x_i^{(1)}\geq 0
               \end{array}\right.\,.
\end{equation}
Here
\begin{equation}\label{eq:kappa}
    \kappa=\Big(\sum_{i:x_i^{(1)}< 0}\Big(x_i^{(1)}\Big)^2\Big)^{1/2}
    \Big(\sum_{i:x_i^{(1)}\geq 0}\Big(x_i^{(1)}\Big)^2\Big)^{-1/2}\,
\end{equation}
ensures that $\mathrm{Cov}(Z^{(1)},Z^{(2)})=0$, and
$\bar{z}^{(2)}=\frac{1}{N}(\sum_{x_i^{(1)}\geq 0}\frac{ \kappa
x_i^{(1)}}{\rho}+ \sum_{ x_i^{(1)} <0} \frac{- x_i^{(1)}}{\rho})$
and $\rho$ are chosen so that the sample mean and variance of
$Z^{(2)}$ are equal to zero and one, respectively. We assume
without loss of generality that $\kappa \geq 1$.

Note that $Z$ depends only on the first column of $X$. Moreover,
each point $z_i$ is just a linear transformation of $x_i^{(1)}$.
In the case of neighborhoods $Z_i$, the situation can be
different. If the first column of $X_i$ is either non-negative or
non-positive, then $Z_i$ is indeed a linear transformation of
$X_i^{(1)}$. However, if $X_i^{(1)}$ is located on both sides of
zero, $Z_i$ is not a linear transformation of $X_i^{(1)}$. Denote
by $N_0$ the set of indices $i$ of neighborhoods $Z_i$ that are
not linear transformations of $X_i^{(1)}$. The number $|N_0|$
depends on the number of nearest neighbors $K$. Recall that for
each neighborhood, we defined the radius $r(i)=\max_{j,k\in
\{0,\ldots,K\}}\norm{x_{i,j}-x_{i,k}}$. Define
$r_{\max}=\max_{i\in N_0} r(i)$ to be the maximum radius of
neighborhoods $i$, such that $i\in N_0$.

\subsection{The embeddings for LEM and DFM} \label{sec:embeddings_LEM_DFM}
So far we have claimed that given the original sample $X$, we
expect the output to resemble $Y$ (see Eq.~\ref{eq:y_i_general}).
However, $Y$ does not satisfy the normalization constraints of
Eq.~\ref{eq:normalization_constraints} for the cases of LEM and
DFM. Define $\hat{Y}$ to be the only affine transformation of $X$
(up to rotation) that satisfies the normalization constraint of
LEM and DFM. When the original sample is given by $X$, we expect
the output of LEM and DFM to resemble $\hat{Y}$. We note that
unlike the matrix $Y$ that was defined in terms of the matrix $X$
only, $\hat{Y}$ depends also on the choice of neighborhoods
through the matrix $D$ that appears in the normalization
constraints.

We define $\hat{Y}$ more formally. Denote
$\widetilde{X}=X-\frac{1} {\textbf{1}'D\textbf{1}}\textbf{1}
\textbf{1}'DX$. Note that $\widetilde{X}$ is just a translation of
$X$ that ensures that $\widetilde{X}'D\textbf{1}=\textbf{0}$. The
matrix $\widetilde{X}'D\widetilde{X}$ is positive definite and
therefore can be presented by $\Gamma \widehat{\Sigma} \Gamma'$
where $\Gamma$ is a $2\times 2$ orthogonal matrix and
\begin{equation*}
\widehat{\Sigma}=\left(
                   \begin{array}{cc}
                     \hat{\sigma}^2 & 0 \\
                     0 & \hat{\tau}^2 \\
                   \end{array}
                 \right)\,,
\end{equation*}
where $ \hat{\sigma}\geq \hat{\tau}$. Define
$\widehat{X}=\widetilde{X}\Gamma$; then
$\widehat{Y}=\widehat{X}\widehat{\Sigma}^{-1/2}$ is the only
affine transformation of $X$ that satisfies the normalization
constraints of LEM and DFM; namely, we have
$\widehat{Y}'D\widehat{Y}=I$ and
$\widehat{Y}'D\textbf{1}=\textbf{0}$.

We define $\widehat{Z}$ similarly to Eq.~\ref{eq:z_i_general},
\begin{equation*}
   \hat{z}_{i}=\left\{ \begin{array}{cc}
                 \left(\frac{\hat{x}_i^{(1)}}{\hat{\sigma}},
                 \frac{- \hat{x}_i^{(1)}}{\hat{\rho}}-\hat{\bar{z}}^{(2)}\right) & \hat{x}_i^{(1)}
                 <0\\ \\
                  \left(\frac{\hat{x}_i^{(1)}}{\hat{\sigma}},
                 \frac{\hat{\kappa} \hat{x}_i^{(1)}}{\hat{\rho}}-\hat{\bar{z}}^{(2)}\right) & \hat{x}_i^{(1)}\geq 0
               \end{array}\right.\,,
\end{equation*}
where $\hat{\kappa}$ is defined by Eq.~\ref{eq:kappa} with respect
to $\widehat{X}$,
$\hat{\bar{z}}^{(2)}=\frac{1}{N}(\sum_{x_i^{(1)}\geq 0}\frac{
d_{ii}\hat{\kappa} x_i^{(1)}}{\rho}+ \sum_{ x_i^{(1)} <0} \frac{-
 d_{ii} x_i^{(1)}}{\rho})$ and $\hat{\rho}^2=\kappa^2\sum_{\hat{x}_i^{(1)}\geq 0}
d_{ii}\left(\hat{x}_i^{(1)}\right)^2+\sum_{\hat{x}_i^{(1)}\leq 0}
d_{ii}\left(\hat{x}_i^{(1)}\right)^2$.

A similar analysis to that of $Y$ and $Z$ can be performed for
$\widehat{Y}$ and $\widehat{Z}$. The same necessary conditions for
success are obtained, with $\sigma$, $\tau$, and $\rho$ replaced
by $\hat{\sigma}$, $\hat{\tau}$, and $\hat{\rho}$, respectively.
In the case where the distribution of the original points is
uniform, the ratio $\frac{\hat{\sigma}}{\hat{\tau}}$ is close to
the ratio $\frac{\sigma}{\tau}$ and thus the necessary conditions
for the success of LEM and DFM are similar to the conditions in
Corollary~\ref{cor:general}.

\subsection{Characterization of the embeddings}
The main result of this section provides necessary conditions for
the success of the \ona. Following Section~\ref{sec:mainExample},
we say that the algorithms fail if $\Phi(Y)>\Phi(Z)$, where $Y$
and $Z$ are defined in Eqs.~\ref{eq:y_i_general}
and~\ref{eq:z_i_general}, respectively. Thus, a necessary
condition for the success of the \ona~is that $\Phi(Y)\leq
\Phi(Z)$.
\begin{thm}\label{thm:general}
Let $X$ be a sample from a two-dimensional domain and let $\psi(X)$
be its embedding in high-dimensional space. Let $Y$ and $Z$ be
defined as above. Then
\begin{equation*}
   \frac{\kappa ^2}{\rho^2}\left( \bar{e}^{(1)}+\frac{|N_0|}{N}c_a
   r_{\max}^2\right) < \frac{\bar{e}^{(2)}}{\tau^2}\quad\Longrightarrow \quad \Phi(Y)>\Phi(Z) \,,
\end{equation*}
where $c_a$ is a constant that depends on the specific algorithm.
For the algorithms LEM and DFM a more restrictive condition can be
defined:
\begin{equation*}
   \frac{\kappa ^2}{\rho^2}\bar{e}^{(1)} < \frac{\bar{e}^{(2)}}{\tau^2} \quad \Longrightarrow \quad \Phi(Y)>\Phi(Z) \,.
\end{equation*}
\end{thm}
For the proof, see Appendix~\ref{sec:appendixGeneralTheorem}.

Adding some assumptions, we can obtain a simpler criterion. First
note that, in general, $\bar{e}^{(1)}$ and $\bar{e}^{(2)}$ should
be of the same order, since it can be assumed that, locally, the
neighborhoods are uniformly distributed. Second, following
Lemma~\ref{lem:absX} (see Appendix~\ref{sec:appendixLemAbsX}),
when $X^{(1)}$ is a sample from a symmetric unimodal distribution
it can be assumed that $\kappa\approx 1$ and
$\rho^2>\frac{\sigma^2}{8}$. Then we have the following corollary:
\begin{cor}\label{cor:general}
Let $X,Y,Z$ be as in  Theorem~\ref{thm:general}. Let
$c=\sigma/\tau$ be the ratio between the variance of the first and
second columns of $X$. Assume that
$\bar{e}^{(1)}<\sqrt{2}\bar{e}^{(2)}$, $\kappa<\sqrt[4]{2}$, and
$\rho^2>\frac{\sigma^2}{8}$. Then
\begin{equation*}
   4\left(1+\frac{|N_0|}{N}\frac{c_a
   r_{\max}^2}{\sqrt{2}\bar{e}^{(2)}}\right)<c \Rightarrow \Phi(Y)>\Phi(Z)\,.
\end{equation*}
For LEM and DFM, we can write
\begin{equation*}
   4<c \Rightarrow \Phi(Y)>\Phi(Z)\,.
\end{equation*}
\end{cor}

We emphasize that both Theorem~\ref{thm:general} and
Corollary~\ref{cor:general} do not state that $Z$ is the output of
the \ona. However, when the difference between the right side and
the left side of the inequalities is large, one cannot expect the
output to resemble the original sample (see
Lemma~\ref{lem:criterion_failure}). In these cases we say that the
algorithms fail to recover the structure of the original domain.

\subsection{Generalization of the results to manifolds of higher dimensions}
The discussion above introduced necessary conditions for the \ona'
success on two-dimensional manifolds embedded in $\R^\D$.
Necessary conditions for success on general $d$-dimensional
manifolds, $d\geq 3$, can also be obtained. We present here a
simple criterion to demonstrate the fact that there are
$d$-dimensional manifolds that the \ona~cannot recover.

Let $X=[X^{(1)},\ldots,X^{(d)}]$ be a $N\times d$ sample from a
$d$-dimensional domain. Assume that the input for the \ona~is given
by $\psi(X)\subset \R^\D$ where $\psi:\R^d\rightarrow\R^\D$ is a
smooth function and $\D\geq d$ is the dimension of the input. We
assume without loss of generality that $X'\textbf{1}=\textbf{0}$ and
that $\textrm{Cov}(X)$ is a diagonal matrix. Let
$Y=X\textrm{Cov}(X)^{-1/2}$. We define the matrix
$Z=[Z^{(1)},\ldots,Z^{(d)}]$ as follows. The first column of $Z$,
$Z^{(1)}$, equals the first column of $Y$, namely,
$Z^{(1)}=Y^{(1)}$. We define the second column $Z^{(2)}$ similarly
to the definition in Eq.~\ref{eq:z_i_general}:
\begin{equation*}
   Z^{(2)}_i=\left\{ \begin{array}{cc}
                 \frac{- x_i^{(1)}}{\rho}-\bar{z}^{(2)} & x_i^{(1)}
                 <0\\ \\
                  \frac{\kappa x_i^{(1)}}{\rho}-\bar{z}^{(2)} & x_i^{(1)}\geq 0
               \end{array}\right.\,,
\end{equation*}
where $\kappa$ is defined as in Eq.~\ref{eq:kappa}, and
$\bar{z}^{(2)}$ and $\rho$ are chosen so that the sample mean and
variance of $Z^{(2)}$ are equal to zero and one, respectively. We
define the next $d-2$ columns of $Z$ by
\begin{equation*}
Z^{(j)}=\frac{Y^{(j)}-\sigma_{2j}Z^{(2)}}{\sqrt{1-\sigma_{2j}^2}}\,;\quad
j=3,\ldots,d\,,
\end{equation*}
where $\sigma_{2j}=Z^{(2)}{}'Y^{(j)}$. Note that
$Z'\textbf{1}=\textbf{0}$ and $\textrm{Cov}(Z)=I$. Denote
$\sigma_{\max}=\max_{j\in\{3,\ldots,d\}}\sigma_{2j}$.

We bound $\Phi(Z)$ from above:
\begin{eqnarray*}
  \Phi(Z) &=&
   \Phi(Y^{(1)})+\Phi(Z^{(2)})+\sum_{i=1}^N
   \left(\frac{1}{1-\sigma_{2j}^2}\right)
   \sum_{j=3}^d\norm{W_i \left(Y_i^{(j)}-\sigma_{2j}Z_i^{(2)}\right)}^2  \\
   &\leq&  \Phi(Y^{(1)})+\Phi(Z^{(2)})+\frac{1}{1-\sigma_{\max}^2}\sum_{i=1}^N
   \sum_{j=3}^d\norm{W_i  Y_i^{(j)}}^2+\frac{\sigma_{\max}^2}{1-\sigma_{\max}^2}\sum_{i=1}^N
   \sum_{j=3}^d\norm{W_i  Z_i^{(2)}}^2\\
   &=&
   \Phi(Y^{(1)})+\frac{1+(d-3)\sigma_{\max}^2}{1-\sigma_{\max}^2}\Phi(Z^{(2)})
   +\frac{1}{1-\sigma_{\max}^2}\sum_{j=3}^d \Phi(Y^{(j)})\,.
\end{eqnarray*}
Since we may write $\Phi(Y)=\sum_{j=1}^d\Phi(Y^{(j)})$, we have
\begin{equation*}
\frac{1+(d-3)\sigma_{\max}^2}{1-\sigma_{\max}^2}\Phi(Z^{(2)})<\Phi(Y^{(2)})+
\frac{\sigma_{\max}^2}{1-\sigma_{\max}^2}\sum_{j=3}^d
\Phi(Y^{(j)})\Rightarrow \Phi(Z)<\Phi(Y)\,.
\end{equation*}
\begin{figure}[!ht]
\vskip 0.2in
\begin{center}
\includegraphics[width=0.65\linewidth]{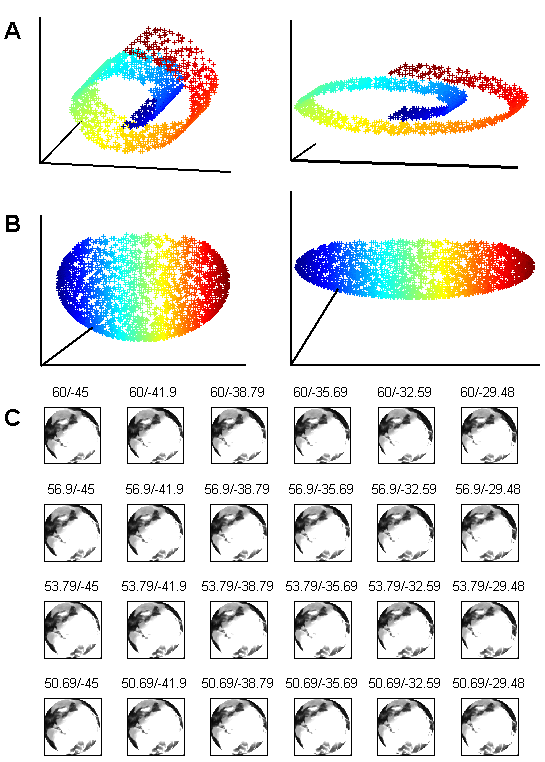}
\caption{The data sets for the first example appear in panel A. In
the left appears the $1600$-point original swissroll and in the
right appears the same swissroll, after its first dimension was
stretched by a factor of $3$. The data for the second example appear
in panel B. In the left appears a $2400$-point uniform sample from
the ``fishbowl", and in the right appears the same ``fishbowl",
after its first dimension was stretched by a factor of $4$. In panel
C appears the upper left corner of the array of $100\times100$ pixel
images of the globe. Above each image we write the elevation and
azimuth. } \label{fig:inputs}
\end{center}
\vskip -0.2in
\end{figure}

When the sample is taken from a symmetric distribution with respect
to the axes, one can expect $\sigma_{\max}$ to be small. In the
specific case of the $d$-dimensional grid, $\sigma_{\max}=0$.
Indeed, $Y^{(j)}$ is symmetric around zero, and all values of
$Z^{(2)}$ appear for a given value of $Y^{(j)}$. Hence, both LEM and
DFM are expected to fail whenever the ratio between the length of
the grid in the first and second coordinates is slightly greater
than $2$ or more, regardless of the length of grid in the other
coordinates, similar to the result presented in
Theorem~\ref{thm:grid}. Corresponding results for the other \ona~can
also be obtained, similar to the derivation of
Corollary~\ref{cor:general}.

\subsection{Numerical results}\label{sec:numericalExamplesGeneral} We
ran all five \ona, along with Isomap, on three data sets. We used
the Matlab implementations written by the algorithms' authors as
provided by~\cite{ManifoldHompage}.

The first data set is a $1600$-point sample from the swissroll as
obtained from~\cite{ManifoldHompage}. The results for the swissroll
are given in Fig.~\ref{fig:swissroll12}, A1-F1. The results for the
same swissroll, after its first dimension was stretched by a factor
$3$, are given in Fig.~\ref{fig:swissroll12}, A2-F2. The original
and stretched swissrolls are presented in Fig.~\ref{fig:inputs}A.
The results for $K=8$ are given in Fig.~\ref{fig:swissroll12}. We
also checked for $K=12,16$; but ``short-circuits"
occur~\citep[see][for a definition and discussion of
``short-circuits"]{IsomapShortCircuits}.

\begin{figure}[!th]
\vskip 0.2in
\begin{center}
\includegraphics[width=0.9\linewidth]{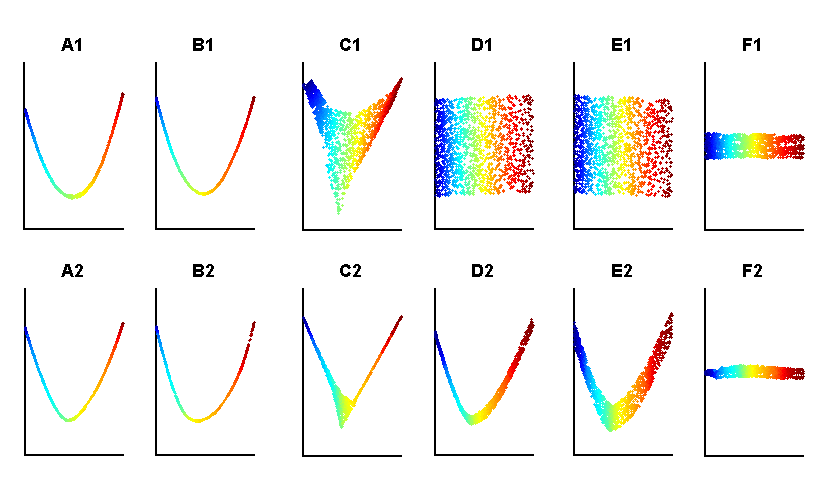}
\caption{The output of LEM on $1600$ points sampled from a
swissroll is presented in A1. The output of LEM on the same
swissroll after stretching its first dimension by a factor of $3$
is presented in A2. Similarly, the outputs of DFM, LLE, LTSA,
HLLE, and Isomap are presented in B1-2, C1-2, D1-2, E1-2, and
F1-2, respectively. We used $K=8$ for all algorithms except DFM,
where we used $\sigma=2$.} \label{fig:swissroll12}
\end{center}
\vskip -0.2in
\end{figure}

The second data set consists of $2400$ points, uniformly sampled
from a ``fishbowl", where a ``fishbowl" is a two-dimensional sphere
minus a neighborhood of the northern pole (see
Fig.~\ref{fig:inputs}B for both the ``fishbowl" and its stretched
version). The results for $K=8$ are given in
Fig.~\ref{fig:fishbowl12}. We also checked for $K=12,16$; the
results are roughly similar. Note that the ``fishbowl" is a
two-dimensional manifold embedded in $\R^3$, which is not an
isometry. While our theoretical results were proved under the
assumption of isometry, this example shows that the \ona~prefer to
collapse their output even in more general settings.

\begin{figure}[!th]
\vskip 0.2in
\begin{center}
\includegraphics[width=0.9\linewidth]{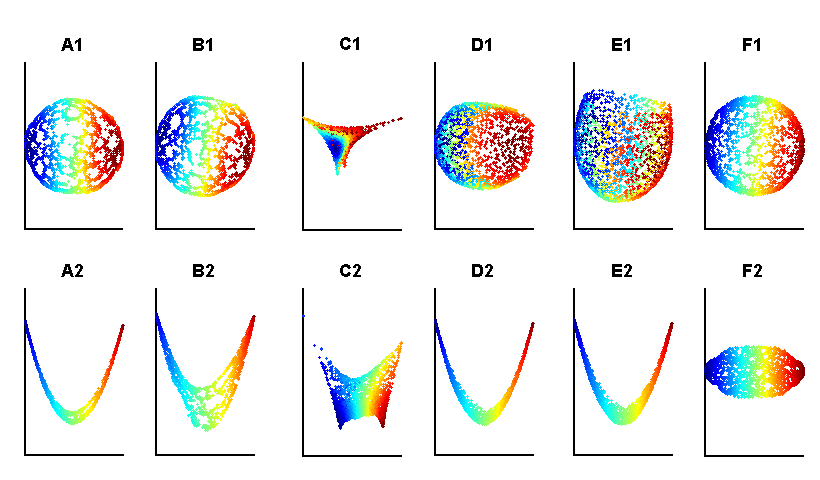}
\caption{The output of LEM on $2400$ points sampled from a
``fishbowl" is presented in A1. The output of LEM on the same
``fishbowl" after stretching its first dimension by a factor of $4$
is presented in A2. Similarly, the outputs of DFM, LLE, LTSA, HLLE,
and Isomap are presented in B1-2, C1-2, D1-2, E1-2, and F1-2,
respectively. We used $K=8$ for all algorithms except DFM, where we
used $\sigma=2$.} \label{fig:fishbowl12}
\end{center}
\vskip -0.2in
\end{figure}

The third data set consists of $900$ images of the globe, each of
$100\times100$ pixels (see Fig.~\ref{fig:inputs}C). The images,
provided by~\cite{earthDataset}, were obtained by changing the
globe's azimuthal and elevation angles. The parameters of the
variations are given by a $30\times 30$ array that contains $-45$ to
$45$ degrees of azimuth and $-30$ to $60$ degrees of elevation. We
checked the algorithms both on the entire set of images and on a
strip of $30\times 10$ angular variations. The results for $K=8$ are
given in Fig.~\ref{fig:earth12}. We also checked for $K=12,16$; the
results are roughly similar.

\begin{figure}[!t]
\vskip 0.2in
\begin{center}
\includegraphics[width=0.9\linewidth]{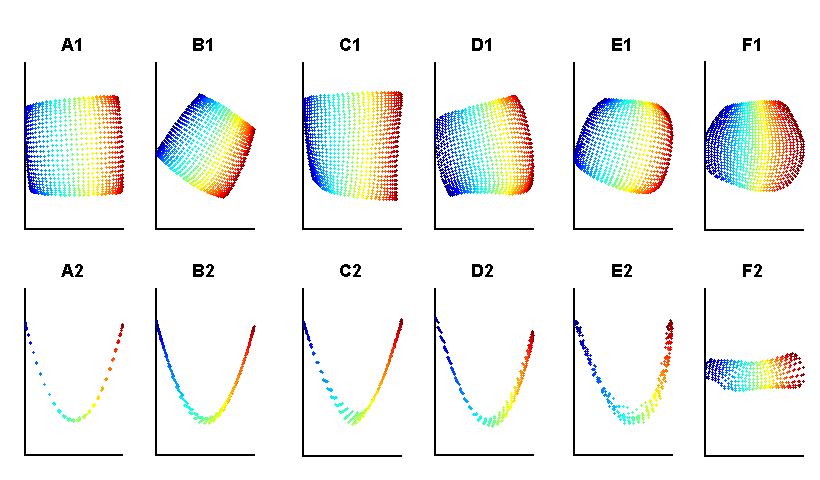}
\caption{The output of LEM on the $30\times 30$ array of the globe
rotation images is presented in A1; the output of LEM on the array
of $30\times 10$ is presented in A2. Similarly, the outputs of
DFM, LLE, LTSA, HLLE, and Isomap are presented in B1-2, C1-2,
D1-2, E1-2, and F1-2 respectively. We used $K=8$ for all
algorithms except DFM, where we chose $\sigma$ to be the root of
the average distance between neighboring points.
\label{fig:earth12}}
\end{center}
\vskip -0.2in
\end{figure}

These three examples, in addition to the noisy version of the
two-dimensional strip discussed in Section~\ref{sec:mainExample}
(see Fig.~\ref{fig:LTSAexample}), clearly demonstrate that when the
aspect ratio is sufficiently large, all the \ona~prefer to collapse
their output.

\section{Asymptotics}\label{sec:asymptotics}
In the previous sections we analyzed the phenomenon of global
distortion obtained by the \ona~on a finite input sample. However,
it is of great importance to explore the limit behavior of the
algorithms, i.e., the behavior when the number of input points
tends to infinity. We consider the question of convergence in the
case of input that consists of a $d$-dimensional manifold embedded
in $\R^\D$, where the manifold is isometric to a convex subset of
Euclidean space. By convergence we mean recovering the original
subset of $\R^d$ up to a non-singular affine transformation.

Some previous theoretical works presented results related to the
convergence issue. \cite{LTSAConvergence1} proved convergence of
LTSA under some conditions. However, to the best of our knowledge,
no proof or contradiction of convergence has been given for any
other of the \ona. In this section we discuss the various
algorithms separately. We also discuss the influence of noise on
the convergence. Using the results from previous sections, we show
that there are classes of manifolds on which the \ona~cannot be
expected to recover the original sample, not even asymptotically.

\subsection{LEM and DFM}
Let $x_1,x_2,\ldots$ be a uniform sample from the two-dimensional
strip $S = [0,L]\times [0,1]$. Let $X_n=[x_1,\ldots,x_n]'$ be the
sample of size $n$. Let $K=K(n)$ be the number of nearest
neighbors. Then when $K\ll n$ there exists with probability one an
$N$, such that for all $n>N$ the assumptions of
Corollary~\ref{cor:general} hold. Thus, if $L>4$ we do not expect
either LEM or DFM to recover the structure of the strip as the
number of points in the sample tends to infinity. Note that this
result does not depend on the number of neighbors or the width of
the kernel, which can be changed as a function of the number of
points $n$, as long as $K \ll n$. Hence, we conclude that LEM and
DFM generally do not converge, regardless of the choice of
parameters.

In the rest of this subsection we present further explanations
regarding the failure of LEM and DFM based on the asymptotic
behavior of the graph Laplacian \citep[see][for details]{belkin}.
Although it was not mentioned explicitly in this paper, it is well
known that the outputs of LEM and DFM are highly related to the
lower non-negative eigenvectors of the normalized graph Laplacian
matrix (see Appendix~\ref{sec:appendixAlgorithms}). It was shown
by~\citet{TowardsBelkin}, \citet{FromGraph}, and
\citet{FromGraphConvergence} that the graph Laplacian operator
converges to the continuous Laplacian operator. Thus, taking a
close look at the eigenfunctions of the continuous Laplacian
operator may reveal some additional insight into the behavior of
both LEM and DFM.

Our working example is the two-dimensional strip $S = [0,L]\times
[0,1]$, which can be considered as the continuous counterpart of the
grid $X$ defined in Section~\ref{sec:grid}. Following~\cite{DFM} we
impose the Neumann boundary condition (see details therein). The
eigenfunctions $\varphi_{i,j}(x_1,x_2)$ and eigenvalues
$\lambda_{i,j}$ on the strip $S$ under these conditions are given by
\begin{equation*}
\varphi_{i,j}(x_1,x_2)=\cos\left(\frac{i\pi}{L}x_1\right)\cos\left(j\pi
x_2\right)\;\;\;\lambda_{i,j}=\left(\frac{i\pi}{L}\right)^2+\left(j\pi\right)^2\;\;\;
\mathrm{for}\;i,j=0,1,2,\ldots\,.
\end{equation*}
When the aspect ratio of the strip satisfies $L>M\in\mathbb{N}$, the
first $M$ non-trivial eigenfunctions are
$\varphi_{i,0},\;i=1,\ldots,M$, which are functions only of the
first variable $x_1$. Any embedding of the strip based on the first
$M$ eigenfunctions is therefore a function of only the first
variable $x_1$. Specifically, whenever $L>2$ the two-dimensional
embedding is a function of the first variable only, and therefore
clearly cannot establish a faithful embedding of the strip. Note
that here we have obtained the same ratio constant $L>2$ computed
for the grid (see Section~\ref{sec:grid} and Figs.~\ref{fig:grid}
and \ref{fig:random}) and not the looser constant $L>4$ that was
obtained in Corollary~\ref{cor:general} for general manifolds.

\subsection{LLE, LTSA and HLLE}
As mentioned in the beginning of this section,
\citet{LTSAConvergence1} proved the convergence of the LTSA
algorithm. The authors of HLLE proved that the continuous manifold
can be recovered by finding the null space of the continuous
Hessian operator \citep[see][Corollary]{HessianEigenMap}. However,
this is not a proof that the algorithm HLLE converges. In the
sequel, we try to understand the relation between
Corollary~\ref{cor:general} and the convergence proof of LTSA.

Let $x_1,x_2,\ldots$ be a sample from a compact and convex domain
$\Omega$ in $\R^2$. Let $X_n=[x_1,\ldots,x_n]'$ be the sample of
size $n$. Let $\psi$ be an isometric mapping from $\R^2$ to
$\R^\D$, where $\D>2$. Let $\psi(X_n)$ be the input for the
algorithms. We assume that there is an $N$ such that for all $n>N$
the assumptions of Corollary~\ref{cor:general} hold. This
assumption is reasonable, for example, in the case of a uniform
sample from the strip $S$. In this case
Corollary~\ref{cor:general} states that $\Phi(Z_n)<\Phi(Y_n)$
whenever
\begin{equation*}
   4\left(1+\frac{|n_0|}{n}\frac{c_a
   r_{\max,n}^2}{\sqrt{2}\bar{e}_n^{(2)}}\right)<c_n\,,
\end{equation*}
where $c_n$ is the ratio between the variance of $X_n^{(1)}$ and
$X_n^{(2)}$ assumed to converge to a constant $c$. The expression
$\frac{|n_0|}{n}$ is the fraction of neighborhoods $X_{i,n}$ such
that $X_{i,n}^{(1)}$ is located on both sides of zero. $r_{\max,n}$
is the maximum radius of neighborhood in $n_0$. Note that we expect
both $\frac{|n_0|}{n}$ and $r_{\max,n}$ to be bounded whenever the
radius of the neighborhoods does not increase. Thus,
Corollary~\ref{cor:general} tells us that if $\{\bar{e}_n^{(2)}\}$
is bounded from below, we cannot expect convergence from LLE, LTSA
or HLLE when $c$ is large enough.

The consequence of this discussion is that a necessary condition for
the convergence of LLE, LTSA and HLLE is that $\{\bar{e}_n^{(2)}\}$
(and hence, from the assumptions of Corollary~\ref{cor:general},
also $\{\bar{e}_n^{(1)}\}$) converges to zero. If the
two-dimensional manifold $\psi(\Omega)$ is not contained in a linear
two-dimensional subspace of $\R^\D$, the mean error
$\bar{e}_n^{(2)}$ is typically not zero due to curvature. However,
if the radii of the neighborhoods tend to zero while the number of
points in each neighborhood tends to infinity, we expect
$\bar{e}_n^{(2)}\rightarrow 0$ for both LTSA and HLLE. This is
because the neighborhood matrices $W_i$ are based on the linear
approximation of the neighborhood as captured by the neighborhood
SVD. When the radius of the neighborhood tends to zero, this
approximation gets better and hence the error tends to zero. The
same reasoning works for LLE, although the use of regularization in
the second step of LLE may prevent $\bar{e}_n^{(2)}$ from converging
to zero (see Section~\ref{sec:algo}).

We conclude that a necessary condition for convergence is that the
radii of the neighborhoods tend to zero. In the presence of noise,
this requirement cannot be fulfilled. Assume that each input point
is of the form  $\psi(x_i)+\varepsilon_i$ where
$\varepsilon_i\in\R^\D$ is a random error that is independent of
$\varepsilon_j$ for $j\neq i$. We may assume that
$\varepsilon_{i}\sim N(0,\alpha^2 I)$, where $\alpha$ is a small
constant. If the radius of the neighborhood is smaller than
$\alpha$, the neighborhood cannot be approximated reasonably by a
two-dimensional projection. Hence, in the presence of noise of a
constant magnitude, the radii of the neighborhoods cannot tend to
zero. In that case, LLE, LTSA and HLLE might not converge, depending
on the ratio $c$. This observation seems to be known also
to~\citeauthor{LTSAConvergence1}, who wrote:
\begin{quote}
``... we assume $\alpha=o(r)$; i.e., we have
$\frac{\alpha}{r}\rightarrow 0$, as $ r \rightarrow 0$.

It is reasonable to require that the error bound ($\alpha$) be
smaller than the size of the neighborhood ($r$), which is reflected
in the above condition. Notice that this condition is also somewhat
nonstandard, since the magnitude of the errors is assumed to depend
on $n$, but it seems to be necessary to ensure the consistency of
LTSA."\footnote{We replaced the original $\tau$ and $\sigma$ with
$r$ and $\alpha$ respectively to avoid confusion with previous
notations. }
 \end{quote}

Summarizing, convergence may be expected when $n\rightarrow \infty$,
if no noise is introduced. If noise is introduced and if
$\sigma/\tau$ is large enough (depending on the level of noise
$\alpha$), convergence cannot be expected (see
Fig.~\ref{fig:LTSAexample}).

\section{Concluding remarks}\label{sec:discussion}

In the introduction to this paper we posed the following question:
Do the normalized-output algorithms succeed in revealing the
underlying low-dimensional structure of manifolds embedded in
high-dimensional spaces? More specifically, does the output of the
normalized-output algorithms resemble the original sample up to
affine transformation?

The answer, in general, is no. As we have seen,
Theorem~\ref{thm:general} and Corollary~\ref{cor:general} show that
there are simple low-dimensional manifolds, isometrically embedded
in high-dimensional spaces, for which the normalized-output
algorithms fail to find the appropriate output. Moreover, the
discussion in Section~\ref{sec:asymptotics} shows that when noise is
introduced, even of small magnitude, this result holds
asymptotically for all the normalized-output algorithms. We have
demonstrated these results numerically for four different examples:
the swissroll, the noisy strip, the (non-isometrically embedded)
``fishbowl", and a real-world data set of globe images. Thus, we
conclude that the use of the normalized-output algorithms on
arbitrary data can be problematic.

The main challenge raised by this paper is the need to develop
manifold-learning algorithms that have low computational demands,
are robust against noise, and have theoretical convergence
guarantees. Existing algorithms are only partially successful:
normalized-output algorithms are efficient, but are not guaranteed
to converge, while Isomap is guaranteed to converge, but is
computationally expensive. A possible way to achieve all of the
goals simultaneously is to improve the existing normalized-output
algorithms. While it is clear that, due to the normalization
constraints, one cannot hope for geodesic distances preservation nor
for neighborhoods structure preservation, success as measured by
other criteria may be achieved. A suggestion of improvement for LEM
appears in~\citet{gerber}, yet this improvement is both
computationally expensive and lacks a rigorous consistency proof. We
hope that future research finds additional ways to improve the
existing methods, given the improved understanding of the underlying
problems detailed in this paper.



\acks{We are grateful to the anonymous reviewers of present and
earlier versions of this manuscript for their helpful suggestions.
We thank an anonymous referee for pointing out errors in the proof
of Lemma~\ref{lem:criterion_failure}. We thank J. Hamm for providing
the database of globe images. This research was supported in part by
Israeli Science Foundation grant and in part by NSF, grant
DMS-0605236.}

\appendix
\section{Detailed proofs and discussions}

\subsection{The equivalence of the algorithms' representations}\label{sec:appendixAlgorithms}
For LEM, note that according to our representation, one needs to
minimize
\begin{equation*}
    \Phi(Y)=\sum_{i=1}^N\norm{W_iY_i}^2_F=\sum_{i=1}^N\sum_{j=1}^Kw_{i,j}\|y_i-y_{i,j}\|^2\,,
\end{equation*}
under the constraints $Y'D\textbf{1}=\textbf{0}$ and $Y'DY=I$.
Define $\hat{w}_{rs}=w_{r,j}$ if $s$ is the $j$-th neighbor of $r$
and zero otherwise. Define $\hat{D}$ to be the diagonal matrix such
that $d_{rr}=\sum_{s=1}^N \hat{w}_{rs}$; note that $\hat{D}=D$.
Using these definitions, one needs to minimize
$\Phi(Y)=\sum_{r,s}\hat{w}_{rs}\|y_r-y_s\|^2$ under the constraints
$Y'\hat{D}\textbf{1}=\textbf{0}$ and $Y'\hat{D}Y=I$, which is the
the authors' representation of the algorithm.

For DFM, as for LEM, we define the weights $\hat{w}_{rs}$. Define
the $N\times N$ matrix $\hat{W}=\left(\hat{w}_{rs}\right)$. Define
the matrix $D^{-1}\hat{W}$; note that this matrix is a Markovian
matrix and that $v^{(0)}\equiv\textbf{1}$ is its eigenvector
corresponding to eigenvalue $1$, which is the largest eigenvalue of
the matrix. Let $v^{(p)}$, $p=1,\ldots,d$ be the next $d$
eigenvectors, corresponding to the next $d$ largest eigenvalues
$\lambda_p$, normalized such that $v^{(p)}{}'Dv^{(p)}=1$. Note that
the vectors $v^{(0)},\ldots,v^{(d)}$ are also the eigenvectors of
$I-D^{-1}W$ corresponding to the $d+1$ lowest eigenvalues. Thus, the
matrix $[v^{(1)},\ldots, v^{(d)}]$ (up to rotation) can be computed
by minimizing $\tr{Y'(D-W)Y}$ under the constraints $Y'DY=I$ and
$Y'D\textbf{1}=\textbf{0}$. Simple computation
shows~\citep[see][Eq.~3.1]{belkin} that
$\tr{Y'(D-W)Y}=\frac{1}{2}\sum_{r,s}\hat{w}_{rs}\|y_r-y_s\|^2$. We
already showed that $\Phi(Y)=\sum_{r,s}\hat{w}_{rs}\|y_r-y_s\|^2$.
Hence, minimizing $\tr{Y'(D-W)Y}$ under the constraints $Y'DY=I$ and
$Y'D\textbf{1}=\textbf{0}$ is equivalent to minimizing $\Phi(Y)$
under the same constraints.  The embedding suggested by~\cite{DFM}
(up to rotation) is the matrix $\left[\lambda_1
\frac{v^{(1)}}{\norm{v^{(1)}}},\ldots,\lambda_d
\frac{v^{(d)}}{\norm{v^{(d)}}}\right]$. Note that this embedding can
be obtained from the output matrix $Y$ by a simple linear
transformation.

For LLE, note that according to our representation, one needs to
minimize
\begin{equation*}
    \Phi(Y)=\sum_{i=1}^N\norm{W_iY_i}^2_F=\sum_{i=1}^N\|y_i-\sum_{j=1}^{K}w_{i,j}y_{i,j}\|^2
\end{equation*}
under the constraints $Y'\textbf{1}=\textbf{0}$ and $\Cov(Y)=I$,
which is the minimization problem given by~\cite{LLE}.

The representation of LTSA is similar to the representation that
appears in the original paper, differing only in the weights'
definition. We defined the weights $W_i$
following~\citet{LTSAConvergence1}, who showed that both
definitions are equivalent.

For HLLE, note that according to our representation, one needs to
minimize
\begin{equation*}
    \Phi(Y)=\sum_{i=1}^N\norm{W_iY_i}^2_F=\sum_{i=1}^N \tr{Y_i'H_i'H_i
    Y_i}
\end{equation*}
under the constraint $\Cov(Y)=I$. This is equivalent (up to a
multiplication by $\sqrt(N)$) to minimizing $\tr{Y'\mathcal{H}Y}$
under the constraint $Y'Y=I$, where $\mathcal{H}$ is the matrix that
appears in the original definition of the algorithm. This
minimization can be calculated by finding the $d+1$ lowest
eigenvectors of $\mathcal{H}$, which is the embedding suggested
by~\cite{HessianEigenMap}.

\subsection{Proof of Lemma~\ref{lem:criterion_failure}}\label{sec:appendixLemCriterionFailure}
We begin by estimating $\Phi(\widetilde{Y})$.
\begin{eqnarray}\label{eq:PhiYTilde}
  \Phi(\widetilde{Y}) &=& \sum_{i=1}^N \norm{W_i Y_i+\eps W_i E_i}_F^2= \sum_{i=1}^N \sum_{j=0}^K\norm{W_i y_{i,j}+\eps W_i e_{i,j}}^2\\
  &\geq& \sum_{i=1}^N \sum_{j=0}^K\left(\norm{W_i y_{i,j}}^2-2\eps |(W_i y_{i,j})' W_i
  e_{i,j}|\right) \nonumber\\
&\geq& \sum_{i=1}^N \sum_{j=0}^K \left((1-4\eps)\norm{W_i
y_{i,j}}^2-4\eps \norm{W_i e_{i,j}}^2\right)\nonumber \\
   &=  & (1-4\eps) \sum_{i=1}^N\norm{W_i Y_i}_F^2-4\eps\sum_{i=1}^N\norm{ W_i E_i}_F^2\nonumber\\
   &\geq & (1-4\eps)\Phi(Y)-4\eps \sum_{i=1}^N\norm{
   W_i}_F^2\norm{E_i}_F^2\,,\nonumber
\end{eqnarray}
where $e_{i,j}$ denotes the $j$-th row of $E_i$.

We bound $\norm{W_i}_F^2$ for each of the algorithms by a constant
$C_a$. It can be shown that for LEM and DFM, $C_a\leq2K$; for LTSA,
$C_a\leq K$; for HLLE $C_a\leq \frac{d(d+1)}{2}$. For LLE in the
case of positive weights $w_{i,j}$, we have $C_a\leq 2$. Thus,
substituting $C_a$ in Eq.~\ref{eq:PhiYTilde}, we obtain
\begin{eqnarray*}
\Phi(\widetilde{Y})&\geq &(1-4\eps) \Phi(Y)-4\eps C_a\sum_{i=1}^N
\sum_{j=0}^K \norm{e_{i,j}}^2 \\
&\geq& (1-4\eps) \Phi(Y)-4\eps C_a S\norm{E}_F^2=(1-4\eps)
\Phi(Y)-4\eps C_a S\,.
\end{eqnarray*}
The last inequality holds true since $S$ is the maximum number of
neighborhoods to which a single observation belongs.

\subsection{Proof of Eq.~\ref{eq:sigma}}\label{sec:appendixSigma}
By definition $\sigma^2=\mathrm{Var}(X^{(1)})$ and hence,
\begin{eqnarray*}
  \sigma^2&=&\frac{1}{N}\sum_{i=-m}^{m}\sum_{j=-q}^{q}\left(x_{ij}^{(1)}\right)^2\\
   &=& \frac{1}{(2m+1)(2q+1)}\sum_{i=-m}^{m}\sum_{j=-q}^{q}i^2 \\
   &=& \frac{2}{2m+1}\sum_{i=1}^{m}i^2 \\
   &=&\frac{2}{2m+1} \frac{(2m+1)(m+1)m}{6}\\
   &=&\frac{(m+1)m}{3}\,.
\end{eqnarray*}
The computation for $\tau$ is similar.

\subsection{Estimation of $F(K)$ and $\widetilde{F}(K)$ for a ball of
radius $r$}\label{sec:appendixRball} Calculation of $\phi(Y_{ij})$
for general $K$ can be different for different choices of
neighborhoods. Therefore, we restrict ourselves to estimating
$\phi(Y_{ij})$ when the neighbors are all the points inside an
$r$-ball in the original grid. Recall that $\phi(Y_{ij})$ for an
inner point is equal to the sum of the squared distance between
$y_{ij}$ and its neighbors. The function
\begin{equation*}
    f(x_1,x_2)=\left(\frac{x_1}{\sigma}\right)^2+\left(\frac{x_2}{\tau}\right)^2
\end{equation*}
agrees with the squared distance for points on the grid, where $x_1$
and $x_2$ indicate the horizontal and vertical distances from
$x_{ij}$ in the original grid, respectively. We estimate
$\phi(Y_{ij})$ using integration of $f(x_1,x_2)$ on $B(r)$, a ball
of radius $r$, which yields
 \begin{equation}\label{eq:PhiYestimationInt}
    \phi(Y_{ij}) \approx \hspace{-0.4cm} \int \limits_{(x_1^2+x_2^2)<r^2}
    \hspace{-0.4cm} f(x_1,x_2)\mathrm{d}x_1
    \mathrm{d}x_2
    =\frac{\pi r^4}{4}\left(\frac{1}{\sigma^2}+\frac{1}{\tau^2}\right)\,.
\end{equation}
Thus, we obtain $F(K)\approx \frac{\pi r^4}{4}$.

We estimate $\phi(Z_{ij})$ similarly. We define the continuous
version of the squared distance in the case of the embedding $Z$ by
\begin{equation*}
    g(x_1,x_2)=x_1^2\left(\frac{1}{\sigma^2}+\frac{4}{\rho^2}\right)\,.
\end{equation*}
Integration yields
\begin{equation}\label{eq:PhiZestimationInt}
    \phi(Z_{ij}) \approx \hspace{-0.4cm} \int  \limits_{(x_1^2+x_2^2)<r^2}
    \hspace{-0.4cm} g(x_1,x_2)\mathrm{d}x_1
     \mathrm{d}x_2=\frac{\pi r^4}{4}\left(\frac{1}{\sigma^2}+\frac{4}{\rho^2}\right)\,.
\end{equation}
Hence, we obtain $\widetilde{F}(K)\approx \frac{\pi r^4}{4} $ and
the relations between Eqs.~\ref{eq:PhiYestimation}
and~\ref{eq:PhiZestimation} are preserved for a ball of general
radius.

For DFM, a general rotation-invariant kernel was considered for
the weights. As with Eqs.~\ref{eq:PhiYestimationInt}
and~\ref{eq:PhiZestimationInt}, the approximations of
$\phi(Y_{ij})$ and $\phi(Z_{ij})$ for the general case with
neighborhood radius $r$ are given by
\begin{equation*}
 \int \limits_{(x_1^2+x_2^2)<r^2}\hspace{-0.1cm}f(x_1,x_2)k(x_1,x_2)
 \mathrm{d}x_1 \mathrm{d}x_2=\left(\pi\hspace{-0.1cm} \int \limits_{0<t<r}
 \hspace{-0.1cm}
 k(t^2)t^3\mathrm{d}t\right)\left(\frac{1}{\sigma^2}+\frac{1}{\tau^2}\right)
\end{equation*}
and
\begin{equation*}
\int \limits_{(x_1^2+x_2^2)<r^2}\hspace{-0.1cm}
g(x_1,x_2)k(x_1,x_2)\mathrm{d}x_1 \mathrm{d}x_2=\left(\pi
\hspace{-0.1cm} \int \limits_{0<t<r} \hspace{-0.1cm}
 k(t^2)t^3\mathrm{d}t\right)\left(\frac{1}{\sigma^2}+\frac{4}{\rho^2}\right)\,.
\end{equation*}
Note that the ratio between these approximations of $\phi(Y_{ij})$
and $\phi(Z_{ij})$ is preserved. In light of these computations it
seems that for the general case of rotation-invariant kernels,
$\phi(Y_{ij})>\phi(Z_{ij})$ for aspect ratio sufficiently greater
than $2$.


\subsection{Proof of Eq.~\ref{eq:rhoGEQsigma}}\label{sec:appendixRho}
Direct computation shows that
\begin{equation*}
    \bar{z}^{(2)}=\frac{(2q+1)2}{N\rho}\sum_{i=1}^{m}(2i)=\frac{2m(m+1)}{(2m+1)\rho}\,.
\end{equation*}
Recall that by definition $\rho$ ensures that $\Var(Z^{(2)})=1$.
Hence,
\begin{eqnarray*}
  1 &=& \frac{1}{N}\sum_{i=-m}^m\sum_{j=-q}^{q}\frac{(2i)^2}{\rho^2}- \left(\bar{z}^{(2)}\right)^2\\
   &=& \frac{2}{2m+1}
   \frac{4m(m+1)(2m+1)}{6\rho^2}-\frac{4m^2(m+1)^2}{(2m+1)^2\rho^2}\\
   &=&  \frac{4m(m+1)}{3\rho^2}-\frac{4m^2(m+1)^2}{(2m+1)^2\rho^2}\,.
    \end{eqnarray*}
Further computation shows that
\begin{equation*}
(m+1)m>\frac{4(m+1)^2m^2}{(2m+1)^2}\,.
\end{equation*}
Hence,
\begin{equation*}
\rho^2> \frac{4(m+1)m}{3}-(m+1)m=\sigma^2\,.
\end{equation*}

\subsection{Proof of
Theorem~\ref{thm:general}}\label{sec:appendixGeneralTheorem} The
proof consists of computing $\Phi(Y)$ and bounding $\Phi(Z)$ from
above. We start by computing $\Phi(Y)$.
\begin{eqnarray*}
  \Phi(Y) &=& \sum_{i=1}^N \norm{W_i Y_i}_F^2= \sum_{i=1}^N \norm{W_i X_i^{(1)}/\sigma}^2+\sum_{i=1}^N \norm{W_i
   X_i^{(2)}/\tau}^2\\
   &=&N \frac{\bar{e}^{(1)}}{\sigma^2}+N\frac{\bar{e}^{(2)}}{\tau^2}\,.
   \end{eqnarray*}
The computation of $\Phi(Z)$ is more delicate because it involves
neighborhoods $Z_i$ that are not linear transformations of their
original sample counterparts.
\begin{eqnarray}
  \Phi(Z) &=& \sum_{i=1}^N \norm{W_i Z_i}_F^2 = \sum_{i=1}^N \norm{W_i Z_i^{(1)}}^2+\sum_{i=1}^N \norm{W_i
   Z_i^{(2)}}^2\nonumber\\
   &=& N\frac{\bar{e}^{(1)}}{\sigma^2}+\hspace{-0.5cm}\sum_{i:x_i^{(1)}<0\; ,\; i\notin N_0}\hspace{-0.5cm}
    \norm{W_i  X_i^{(1)}/\rho}^2
   +\hspace{-0.5cm}\sum_{i:x_i^{(1)}>0 \;,\; i\notin N_0}\hspace{-0.5cm} \norm{\kappa W_i
X_i^{(1)}/\rho}^2+\sum_{i\in N_0} \norm{W_i  Z_i^{(2)}}^2\label{eq:BoundZgeneralLEM}\\
   &<& N\frac{\bar{e}^{(1)}}{\sigma^2}+N\frac{\kappa^2 \bar{e}^{(1)}}{\rho^2}+\sum_{i\in N_0} \norm{W_i
   Z_i^{(2)}}^2\,.\label{eq:BoundZgeneral}
   \end{eqnarray}
Note that $\max_{j,k\in\{0,\ldots,K\}}\norm{z_{i,j}-z_{i,k}}\leq
\kappa r(i)/\rho $. Hence, using
Lemma~\ref{lem:WeightsBoundOnNeighborhoods} we get
\begin{equation}\label{eq:BoundWiZigeneral}
\norm{W_i Z_i^{(2)}}^2 < \frac{c_a \kappa^2 r(i)^2}{\rho^2}\,,
\end{equation}
where $c_a$ is a constant that depends on the specific algorithm.
Combining Eqs.~\ref{eq:BoundZgeneral} and~\ref{eq:BoundWiZigeneral}
we obtain
\begin{equation*}
    \Phi(Z)<N\frac{\bar{e}^{(1)}}{\sigma^2}+N\frac{\kappa^2
    \bar{e}^{(1)}}{\rho^2}+|N_0|c_a
    r_{\max}^2\frac{\kappa^2}{\rho^2}\,.
\end{equation*}
In the specific case of LEM and DFM, a tighter bound can be obtained
for $\norm{W_i Z_i^{(2)}}^2$. Note that for LEM and DFM
\begin{eqnarray*}
    \norm{W_i Z_i^{(2)}}^2
    &=&\sum_{j=1^K}w_{i,j}(z_i^{(2)}-z_{i,j}^{(2)})^2\\
    &\leq & \sum_{j=1}^K w_{i,j}\frac{\kappa^2}{\rho^2}( x_i^{(2)}-x_{i,j}^{(2)})^2=
    \frac{\kappa^2}{\rho^2}e_{i}^{(1)}\,.
\end{eqnarray*}
Combining Eq.~\ref{eq:BoundZgeneralLEM} and the last inequality we
obtain in this case that
\begin{equation*}
    \Phi(Z)\leq N\frac{\bar{e}^{(1)}}{\sigma^2}+N\frac{\kappa^2
    \bar{e}^{(1)}}{\rho^2}\,,
\end{equation*}
which completes the proof.

\subsection{Lemma~\ref{lem:WeightsBoundOnNeighborhoods}}
\begin{lem}\label{lem:WeightsBoundOnNeighborhoods}
Let $X_i=[x_i,x_{i,1},\dots,x_{i,K}]'$ be a local
neighborhood. Let 
$r_i=\max_{j,k}\norm{x_{i,j}-x_{i,k}}$. Then
\begin{equation*}
    \norm{W_i X_i}^2_F<c_a r_i^2\,,
\end{equation*}
where $c_a$ is a constant that depends on the algorithm.
\end{lem}
\begin{proof}
We prove this lemma for each of the different algorithms separately.
\begin{itemize}
  \item LEM and DFM:
    \begin{equation*}
      \norm{W_i X_i}^2_F = \sum_{j=1}^K w_{i,j}\norm{x_{i,j}-x_{i}}^2
     \leq\left(\sum_{j=1}^K w_{i,j}\right)r_i^2 \leq K r_i^2\,,
  \end{equation*}
  where the last inequality holds since $w_{i,j}\leq 1$. Hence
  $c_a=K$.
  \item LLE:
  \begin{eqnarray*}
      \norm{W_i X_i}^2_F &=& \Big|\Big|\sum_{j=1}^K w_{i,j}(x_{i,j}-x_i)\Big|\Big|^2
      \leq \Big|\Big|\frac{1}{K} \sum_{j=1}^K (x_{i,j}-x_i)\Big|\Big|^2\\
     &\leq&\frac{1}{K^2}\sum_{j=1}^K\norm{x_{i,j}-x_i}^2\leq \frac{r_i^2}{K}\,,
  \end{eqnarray*}
where the first inequality holds since $w_{i,j}$ were chosen to
minimize $\norm{\sum_{j=1}^K \tilde{w}_{i,j}(x_{i,j}-x_i)}^2$. Hence
$c_a=1/K$.
  \item LTSA:
\begin{eqnarray*}
   \norm{W_i X_i}^2_F& =& \norm{(I-P_iP_i') H X_i}^2_F \leq \norm{(I-P_iP_i')}^2_F \norm{H
   X_i}^2_F\\& \leq& K \sum_j \norm{x_{i,j}-\bar{x_i}}^2 \leq  K^2
   r_i^2\,.
     \end{eqnarray*}
     The first equality is just the definition of $W_i$ (see Sec.~\ref{sec:algo}).
     The matrix $I-P_iP_i'$ is a projection matrix and its square norm is the dimension of
     its range, which is smaller than $K+1$. Hence $c_a=K^2$.
  \item HLLE:
   \begin{equation*}
   \norm{W_i X_i}^2_F = \norm{W_i H X_i}^2_F \leq \norm{W_i}^2_F \norm{H
   X_i}^2_F \leq \frac{d(d+1)}{2}(K+1) r_i^2\,.
     \end{equation*}
     The first equality holds since $W_i H= W_i
     (I-\frac{1}{K}\textbf{11}')=W_i$, since the rows of $W_i$ are
     orthogonal to the vector $\textbf{1}$ by definition (see
     Sec.~\ref{sec:algo}). Hence $c_a=\frac{d(d+1)}{2}(K+1)$.
\end{itemize}
\end{proof}
\subsection{Lemma~\ref{lem:absX}}\label{sec:appendixLemAbsX}
\begin{lem}\label{lem:absX}
Let $X$ be a random variable symmetric around zero with unimodal
distribution. Assume that $\Var(X)=\sigma^2$. Then
$\Var(|X|)\geq\frac{\sigma^2}{4}$.
\end{lem}
\begin{proof}
First note that that the equality holds for $X \sim
U(-\sqrt{3}\sigma,\sqrt{3}\sigma)$, where $U$ denotes the uniform
distribution. Assume by contradiction that there is a random
variable $X$, symmetric around zero and with unimodal distribution
such that $\Var(|X|)<\frac{\sigma^2}{4}-\eps$, where $\eps>0$.
Since $\Var(|X|)=E(|X|^2)-E(|X|)^2$, and
$E(|X|^2)=E(X^2)=\Var(X)=\sigma^2$, we have
$E(|X|)^2>\frac{3\sigma^2}{4}+\eps$.

We approximate $X$ by $X_n$, where $X_n$ is a mixture of uniform
random variables, defined as follows. Define
$X_n\sim\sum_{i=1}^{n}p_i^{n} U(-c_i^{n},c_i^{n})$ where
$p_i^{n}>0$, $\sum_{i=1}^n p_i^{n}=1$. Note that $E(X_n)=0$ and
that $ \Var(X_n) =\sum_{i=1}^n p_i^n(c_i^n)^2/3$. For large enough
$n$, we can choose $p_i^{n}$ and $c_i^{n}$ such that
$\Var(X_n)=\sigma^2$ and $E(|X-X_n|)<\frac{\eps}{2E(|X|)}$.

Consider the random variable $|X_n|$. Note that using the
definitions above we may write $|X_n|=\sum_{i=1}^n p_{i}^{n}
U(0,c_i^{n})$, hence $E(|X_{n}|)=\frac{1}{2}\sum_{i=1}^{n} p_i^{n}
c_i^{n}$. We bound this expression from below. We have
\begin{eqnarray}\label{eq:Ex_n}
    E(|X_n|)^2&=&E(|X_n-X+X|)^2 \geq (E(|X|)-E(|X_n-X|))^2\\ &\geq&
    E(|X|)^2-2E(|X|)E(|X_n-X|) > \frac{3\sigma^2}{4}\,.\nonumber
\end{eqnarray}

Let $X_{n-1}=\sum_{i=1}^{n-1}p_i^{n-1} U(-c_{i}^{n-1},c_{i}^{n-1})$
where
\begin{equation*}
    p_i^{n-1}=\left\{\begin{array}{ll}
               p_i^{n} & i<n-1 \\
               p_{n-1}^{n}+p_n^{n} & i=n-1
             \end{array}\right.\,,
\end{equation*}
and
\begin{equation*}
c_{i}^{n-1}=\left\{\begin{array}{ll}
               \;c_i^{n} & i<n-1 \\
               \sqrt{\left((c_{n-1}^{n}\right)^2+\left(c_n^{n}\right)^2} & i=n-1
             \end{array}\right.\,.
\end{equation*}
Note that $\textrm{Var}(X_{n-1})=\sigma^2$ by construction and
$X_{n-1}$ is symmetric around zero with unimodal distribution.
Using the triangle inequality we obtain
\begin{equation*}
E(|X_{n-1}|)=\frac{1}{2}\sum_{i=1}^{n-1} p_i^{n-1} c_i^{n-1} \geq
\frac{1}{2}\sum_{i=1}^n p_i^{n} c_i^{n}=E(|X_{n}|)\,.
\end{equation*}
Using the same argument recursively, we obtain that
$E(|X_{1}|)\geq E(|X_{n}|)$. However, $X_1 \sim
U(-\sqrt{3}\sigma,\sqrt{3}\sigma)$ and hence
$E(|X_1|)^2=\frac{3\sigma^2}{4}$. Since by Eq.~\ref{eq:Ex_n},
$E(|X_n|)^2>\frac{3\sigma^2}{4}$ we have a contradiction.
\end{proof}

\end{document}